\documentclass{article}

\usepackage{microtype}
\usepackage{graphicx}
\usepackage{subcaption}
\usepackage{booktabs} 
\usepackage{multirow}
\usepackage{multicol}
\usepackage{hyperref, xurl}
\usepackage{algorithm}
\usepackage[noend]{algorithmic}
\usepackage{pifont}
\usepackage{multicol}

\makeatletter
\newcommand*\@dblLabelI {}
\newcommand*\@dblLabelII {}
\newcommand*\@dblequationAux {}

\def\@dblequationAux #1,#2,%
    {\def\@dblLabelI{\label{#1}}\def\@dblLabelII{\label{#2}}}

\newcommand*{\doubleequation}[3][]{%
    \par\vskip\abovedisplayskip\noindent
    \if\relax\detokenize{#1}\relax
       \let\@dblLabelI\@empty
       \let\@dblLabelII\@empty
    \else 
       \@dblequationAux #1,%
    \fi
    \makebox[0.5\linewidth-1.5em]{%
     \hspace{\stretch2}%
     \makebox[0pt]{$\displaystyle #2$}%
     \hspace{\stretch1}%
    }%
    \makebox[0.5\linewidth-1.5em]{%
     \hspace{\stretch1}%
     \makebox[0pt]{$\displaystyle #3$}%
     \hspace{\stretch2}%
    }%
    \makebox[3em][r]{(%
  \refstepcounter{equation}\theequation\@dblLabelI, 
  \refstepcounter{equation}\theequation\@dblLabelII)}%
  \par\vskip\belowdisplayskip
}

\usepackage[accepted]{icml2024}
\usepackage{amsmath}
\usepackage{amssymb}
\usepackage{mathtools}
\usepackage{amsthm}
\usepackage{wrapfig}

\usepackage[capitalize, noabbrev]{cleveref}

\theoremstyle{plain}
\newtheorem{theorem}{Theorem}[section]
\newtheorem{example}[theorem]{Example}

\newtheorem{lemma}[theorem]{Lemma}
\newtheorem{corollary}[theorem]{Corollary}
\theoremstyle{definition}
\newtheorem{definition}[theorem]{Definition}

\theoremstyle{remark}
\newtheorem{remark}[theorem]{Remark}

\usepackage[textsize=tiny]{todonotes}

\newcommand{\GG}{\mathcal{G}}

\newcommand{\MM}{\mathcal{M}}

\newcommand{\PP}{\mathcal{P}}

\DeclareMathOperator{\NGASS}{NG}

\newcommand{\R}{\mathbb{R}}

\newcommand{\mB}{\mathbf{B}}
\newcommand{\tmB}{\Tilde{\mathbf{B}}}
\newcommand{\mA}{\mathbf{A}}
\newcommand{\tmA}{\Tilde{\mathbf{A}}}

\newcommand{\indep}{\perp \!\!\! \perp}
\newcommand{\dsep}{\perp \!}
\usepackage{paralist}

\DeclareMathOperator{\doop}{do}

\DeclareMathOperator*{\argmin}{arg\,min}

\def\newop#1{\expandafter\def\csname #1\endcsname{\mathop{\rm
#1}\nolimits}}

\newop{Inv}
\newop{conv}
\newop{pa}
\newop{de}
\newop{nd}
\newop{GL}
\newop{O}
\newop{ch}
\newop{CS}
\newop{diag}
\newop{Var}
\newop{top}
\newop{sib}
\newop{an}

\icmltitlerunning{Causal Effect Identification in LiNGAM Models with Latent Confounders}

\begin{document}

\twocolumn[
\icmltitle{Causal Effect Identification in LiNGAM Models with Latent Confounders}



\icmlsetsymbol{equal}{*}

\begin{icmlauthorlist}
\icmlauthor{Daniele Tramontano}{tum}
\icmlauthor{Yaroslav Kivva}{epfl}
\icmlauthor{Saber Salehkaleybar}{liacs}
\icmlauthor{Mathias Drton}{tum,mcml}
\icmlauthor{Negar Kiyavash}{epfl}
\end{icmlauthorlist}

\icmlaffiliation{tum}{Technical University of Munich, Munich, Germany}
\icmlaffiliation{mcml}{Munich Center for Machine Learning, Munich, Germany}
\icmlaffiliation{epfl}{Ecole Polytechnique Fédérale de Lausanne, Lausanne, Switzerland}
\icmlaffiliation{liacs}{Leiden Institute of Advanced Computer Science, Leiden University, Netherlands}

\icmlcorrespondingauthor{Daniele Tramontano}{daniele.tramontano@tum.de}

\icmlkeywords{Machine Learning, ICML}
\vskip 0.3in
]



\printAffiliationsAndNotice{}  

\begin{abstract}
We study the generic identifiability of causal effects in linear non-Gaussian acyclic models (LiNGAM) with latent variables. We consider the problem in two main settings: When the causal graph is known a priori, and when it is unknown. In both settings, we provide a complete graphical characterization of the identifiable direct or total causal effects among observed variables. Moreover, we propose efficient algorithms to certify the graphical conditions. Finally, we propose an adaptation of the reconstruction independent component analysis (RICA) algorithm that estimates the causal effects from the observational data given the causal graph. Experimental results show the effectiveness of the proposed method in estimating the causal effects.
\end{abstract}

\section{Introduction}
\label{sec:intro}
Predicting the impact of an unseen intervention in a system is a crucial challenge in many fields, such as medicine \cite{sanchez:2022, michoel:2023}, policy evaluation \cite{athey:2017}, fair decision-making \cite{kilbertus:2017}, and finance \cite{de:2023}. 
Randomized experiments form the gold standard for addressing this challenge, but are often infeasible due to ethical concerns or prohibitively high costs.
To tackle situations in which only observational data are available, one needs to make additional assumptions on the underlying causal system. The field of causal inference strives to formalize such assumptions.
One notable approach in causal inference is modeling causal relationships through structural equation models (SEM) \citep{pearl:2009}. 
In this framework, a random vector is associated with a directed acyclic graph (DAG). Each vector component is associated with a node in the graph and is a function of the random variables corresponding to its parents in the graph and its corresponding exogenous noise.

In this paper, we mainly focus on the identifiability of causal effects in the important subclass of linear SEM and characterize graphically which causal effects can be uniquely determined from the observational data.
When the exogenous noises in a linear SEM are Gaussian, the entire information about the model is contained in the covariance matrix among the variables, with all the higher-order moments of the distribution being uninformative. This entails that the causal structure as well as other causal quantities are often not identifiable from mere observational data. The most prominent instance in the context of causal structure learning is that the causal graph is identifiable only up to an equivalence class 
\citep[e.g.,][\S10]{drton:2018}.
This motivated the widespread use of the linear non-Gaussian additive noise model (LiNGAM), where the exogenous noises are non-Gaussian. 

The seminal work of \citet{shimizu:2006} showed that in the setting of LiNGAM, the true underlying causal graph is uniquely identifiable when all the variables are observed. Since then, a rich literature on this topic has emerged, focusing mainly on the identification and the estimation of the causal graph; see, e.g., \citet{adams:2021,shimizu:2022, yang:2022, wang:2023, wang:cd:2023} for recent results that allow for the presence of hidden variables. Indeed, linear models remain the backbone of problem abstraction in many scientific disciplines, because they offer simple qualitative interpretations and can be learned with moderate sample sizes \citep[Principle 1]{pe:2011}. In particular, the LiNGAM model finds application in diverse scientific fields, such as Neuroscience \cite{chiyohara:2023}, Economics \cite{ciarli:2023}, or Epidemiology \cite{barrera:2022}.

Within the LiNGAM literature, causal effect identification has received less attention; only few recent work  \citep{kivva:2023, shuai:2023} have exploited the non-Gaussianity to provide identification formulas that work for specific causal graphs. The only graphical criteria for identification are given in \citet{salehkaleybar:2020, yang:2022}. The main drawback of the aforementioned papers is that they target simultaneous recovery of all the causal effects. However, in many applications, we are interested in causal effects of only some subset of variables on others. Indeed, it may be the case that some
causal effects are identifiable while others are not; see \cref{fig:proxy} for an example in the context of proxy variable graphs. 
In this paper, we provide necessary and sufficient graphical conditions for the generic identifiability (see \cref{subsec:ident} for the exact definition) of  direct and total causal effects between a given pair of observed variables in a LiNGAM.

\subsection{Contribution}
\label{subsec:contr}

Our main contributions are as follows:
\begingroup
\setlength{\itemindent}{-10pt}
\setlength{\leftmargini}{0pt}
\begin{itemize}
    \item We provide {necessary and sufficient} graphical criteria for the generic identifiability of the causal effect, both when the causal graph is known a priori (\cref{subsec:ceid:unknown:graph}) and when it is unknown (\cref{subsec:ceid:known:graph}).
    \item We propose sound and complete algorithms that check our criteria in polynomial time in the size of the graph, in both considered settings (\cref{subsec:cert}).
    \item For practical estimation of the effect of interest, we propose an adaptation of the RICA algorithm for Independent Component Analysis \citep{le:2011}. Experimental results show that the proposed method can provide better estimates of causal effects when compared with previous work
    (\cref{subsec:est:alg}).
    \end{itemize}
\endgroup
\section{Problem Definition}
\label{sec:prob:def}
\subsection{Notation}
\label{subsec:not}
A \emph{directed graph} is a pair $\GG=(\mathcal{V},E)$ where $\mathcal{V}:=\{1,\dots,p\}$ is the set of nodes and $E\subseteq \{(i,j)\mid i, j\in \mathcal{V},\, i\neq j\}$ is the set of edges. We denote a pair $(i,j)\in E$ as $i\to j$. 

A (directed) path from node $i$ to node $j$ in $\GG$  is a sequence of nodes $\pi=(i_1=i,\dots,i_{k+1}=j)$ such that $i_s\to i_{s+1}\in E$ for $s\in\{1,\dots,k\}$.  
A cycle in $\GG$ is a path from a node $i$ to itself. 
A \emph{Directed Acyclic Graph} (DAG) is a directed graph without cycles.
If $i\to j\in E$, we say that $i$ is a parent of $j$, and $j$ is a child of $i$. If there is a path from $i$ to $j$ in $\GG$, we say that $i$ is an ancestor of $j$ and $j$ is a descendant of $i$. The sets of parents, children, ancestors, and descendants of a given node $i$ are denoted by $\pa(i), \ch(i), \an(i)$, and $\de(i)$, respectively.
In our work, we distinguish between observed and latent variables by partitioning the nodes in two sets $\mathcal{V} = \mathcal{O}\cup \mathcal{L}$,
of respective sizes $p_o$ and $p_l$.
Moreover, we define the set of observed descendants of a node $i$
as $\de_o(i):=(\de(i)\cup\{i\})\cap\mathcal{O}$.

We write vectors and matrices in boldface. The entry $(i,j)$ of a matrix $\mA$ is denoted by $[\mA]_{i,j}$. 
Let $I, J$ be subsets of the row and column sets of $\mA$, respectively. We denote the submatrix containing only the rows in $I$ and the columns in $J$ as $[\mA]_{I, J}$.
For a permutation $\sigma$,  $\mathbf{P}_{\sigma}$ denotes the associated permutation matrix.

\subsection{Model}
\label{subsec:model}
Let $\GG=(\mathcal{V},E)$ be a \emph{fixed} DAG on $p$ nodes.  In a fixed probability space, let 
$\mathbf{V} = (V_0,\dots, V_p)$ be a random vector taking values in $\R^p$ and satisfying the following structural equation model:
\begin{equation}
\label{eq:sem:1}
    \mathbf{V}=\mathbf{AV}+\mathbf{N}=\mathbf{BN},
\end{equation}
where $[\mathbf{A}]_{j,i}=0$ if $i\to j\notin E$,  $\mathbf{B}:=(\mathbf{I}-\mathbf{A})^{-1}$, and the enteries of the exogenous noise vector $\mathbf{N}$ are assumed to be jointly independent and \emph{non-Gaussian}. $\mathbf{V}$ is partitioned to $[\mathbf{V}_o,\mathbf{V}_l]$, where 
$\mathbf{V}_o$ is observed of dimension $p_o$, while $\mathbf{V}_l$ is latent and of dimension $p_l$. 
We can rewrite \eqref{eq:sem:1} as
\begin{equation*}
    \begin{bmatrix}
        \mathbf{V}_o\\
        \mathbf{V}_l
    \end{bmatrix}=\begin{bmatrix}
        \mathbf{A}_{o,o} & \mathbf{A}_{o,l}\\
        \mathbf{A}_{l,o} & \mathbf{A}_{l,l}
    \end{bmatrix}\begin{bmatrix}
        \mathbf{V}_o\\
        \mathbf{V}_l
    \end{bmatrix}+
        \begin{bmatrix}
        \mathbf{N}_o\\
        \mathbf{N}_l
    \end{bmatrix},
\end{equation*}
which implies that the observed random vector satisfies
\begin{equation}
\label{eq:sem:3}
    \mathbf{V}_o=    
    \mathbf{B}'\mathbf{N}=\begin{bmatrix}
        \mathbf{B}_o & \mathbf{B}_l
    \end{bmatrix}\begin{bmatrix}
        \mathbf{N}_o\\
        \mathbf{N}_l
    \end{bmatrix}
\end{equation}
with $\mathbf{B}':=[(\mathbf{I-A})^{-1}]_{\mathcal{O},\mathcal{V}}$. We refer to this model as the latent variable LiNGAM (lvLiNGAM).\footnote{Although  our results are presented for the lvLiNGAM model, our analysis only relies on the identifiability of the mixing matrix $\mathbf{B}'$ up to permutation and scaling. This can be also achieved in other settings as explained in \citet[\S2.3]{yang:2022} or \citet[\S3]{adams:2021}. Hence, our results also hold in these settings.}

\citet[\S3]{salehkaleybar:2020} show that the matrix $\mathbf{B}'$ can be expressed as follows:
\begin{equation}
    \label{eq:b:matrix}
    \mathbf{B}_o = (\mathbf{I}-\mathbf{D})^{-1},\quad \mathbf{B}_l = (\mathbf{I}-\mathbf{D})^{-1}\mathbf{A}_{o,l}(\mathbf{I}-\mathbf{A}_{l,l})^{-1},
\end{equation}
with $\mathbf{D}=\mathbf{A}_{o,o}+\mathbf{A}_{o,l}(\mathbf{I}-\mathbf{A}_{l,l})^{-1}\mathbf{A}_{l,o}$. The matrices $\mathbf{B}'$ and $\mathbf{D}$ contain information on the interventional distributions of $\mathbf{V}_o$. In particular,\footnote{Please see \citet[\S3]{pearl:2009} for the definition of \emph{do} intervention.}
\begin{equation}
    \begin{aligned}
    \label{eq:ce:of:int}
        [\mathbf{B}']_{i,j}&=\frac{\partial \mathbb{E}(V_i\mid \doop(V_j))}{\partial V_j},\\
        [\mathbf{D}]_{i,j}&=\frac{\partial \mathbb{E}(V_i\mid \doop(V_{\pa(i)}))}{\partial V_j}.
    \end{aligned}
\end{equation}
In other words, $[\mathbf{B}']_{j, i}$ is the average total causal effect of $j$ on $i$, while $[\mathbf{D}]_{j, i}$ is the average causal effect of $j$ on $i$ that is not mediated by other observed nodes. With slight abuse of terminology, we refer to the entries of $\mathbf{B}'$ as \emph{total causal effects} and to those of $\mathbf{D}$ as \emph{direct causal effects}.

\citet{hoyer:2008} show that for any lvLiNGAM model, an associated \emph{canonical model} exists, in which, in the corresponding graph, all the latent nodes have at least two children and have no parents. We refer to the graph corresponding to a canonical model as a canonical graph. The original and the associated canonical model are \emph{observationally} and \emph{causally} equivalent \citep[\S3]{hoyer:2008}. Subsequently, without loss of generality, we will assume our model is canonical in this sense. \citet[Cor.~11]{salehkaleybar:2020} proved that the number of latent variables is identifiable in canonical models. 

\begin{remark}
    Throughout the paper, we assume that the number of latent variables $p_l$ is known. 
\end{remark}

In canonical models, $\mathbf{A}_{l,o}=\mathbf{A}_{l,l}=\mathbf{0}$, and in particular 
\begin{equation}
    \label{eq:b:can:matrix}
    \mathbf{B}_o = (\mathbf{I}-\mathbf{A}_{o,o})^{-1},\qquad \mathbf{B}_l = (\mathbf{I}-\mathbf{A}_{o,o})^{-1}\mathbf{A}_{o,l}.
\end{equation}
For every canonical $\GG$, let $\R_\mathbf{A}^\GG$ be the set of all $p\times p$ real matrices $\mathbf{A}$ such that $[\mathbf{A}]_{i,j}=0$ if $j\to i\notin\GG$. Let $\R^\GG$ be the set of all $p_o\times p$ matrices, $\mathbf{B}'=[\mathbf{B}_{o},\mathbf{B}_{l}]$ that can be obtained from a matrix $\mathbf{A}\in\R_\mathbf{A}^\GG$ from \eqref{eq:b:can:matrix}. Let $\NGASS^p$ be the set of $p$ dimensional, non-degenerate, jointly independent \emph{non-Gaussian} random vectors, and let $\MM(\GG)$ be the set of all $p_o$ dimensional random vectors that can be expressed according to \eqref{eq:sem:3}, where the matrix $\mathbf{B}'\in\R^\GG$. 

\subsection{Identifiability}
\label{subsec:ident}

According to \eqref{eq:sem:3}, the mechanism generating the observational distribution, i.e., the probability distribution of  $\mathbf{V}_o$, only depends on the matrix $\mathbf{B}'$ and the exogenous noise vector $\mathbf{N}$.  Therefore, the parameters of interest such as the causal effects in \eqref{eq:ce:of:int}, are functions $\phi(\mathbf{B}')$.  As  generally there are multiple pairs $(\mathbf{B}',\mathbf{N})$, with different $\mathbf{B}'$s, that generate the same observational distribution, it is important
to clarify whether a parameter $\phi$
is \emph{identifiable}.  This holds if  $\phi(\mathbf{B}')$ takes the same value in all considered pairs $(\mathbf{B}',\mathbf{N})$ generating the same observational distribution.  Adding further assumptions about the causal graph may limit the data-generating mechanisms compatible with the observational data, leading to more parameters becoming identifiable.

In this paper, we study the identification of causal effects in two scenarios: when the causal graph $\GG$ is known and when it is not.
The remainder of the section provides a formal description of the resulting parameter identification problems in the lvLiNGAM setting.

    Fix a DAG $\GG$.  The observational random vector is obtained from the noise and the matrix $\mathbf{B}'$ according to the mapping
    \begin{equation}
    \label{eq:phi:map}
        \begin{aligned}
            \Phi_{\GG}: \R^{\GG}\times\NGASS^{p}&\xrightarrow{}\MM(\GG)\\
             (\mathbf{B}',\mathbf{N}) &\mapsto\mathbf{B}'\mathbf{N}=\begin{bmatrix}
        \mathbf{B}_o & \mathbf{B}_l
    \end{bmatrix}\begin{bmatrix}
        \mathbf{N}_o\\
        \mathbf{N}_l
    \end{bmatrix}.
        \end{aligned}
    \end{equation}
    Let $\phi$ be a real-valued function on $\R^\GG$. We say that the parameter $\phi$ is \emph{globally identifiable without knowledge of the graph} if for every 
    pair $(\mathbf{B}',\mathbf{N})\in\R^{\GG}\times\NGASS^{p}$ and the associated observed random vector
    $\mathbf{V}_o:=\Phi_\GG(\mathbf{B}',\mathbf{N})\in\MM(\GG)$, there does not exist a DAG $\Tilde{\GG}$ and a pair $(\Tilde{\mathbf{B}'}, \Tilde{\mathbf{N}})\in\R^{\tilde{\GG}}\times \NGASS^p$
    such that
    \begin{align}
        \label{eq:id:without:2}
        \Phi_{\tilde{\GG}}(\Tilde{\mathbf{B}'}, \Tilde{\mathbf{N}}) =_d \mathbf{V}_o \ \text{ but } \ \phi(\mathbf{B}')\neq \phi(\tilde{\mathbf{B}}').
    \end{align}
    Here, $=_d$ denotes equality in distribution of two random vectors. Note that the above definition allows for consideration of any $\tilde{\GG}$, possibly, different from $\GG$.

    In contrast, we say that $\phi$ is \emph{globally identifiable with knowledge of the graph} if for every 
    pair $(\mathbf{B}',\mathbf{N})\in\R^{\GG}\times\NGASS^{p}$ and the associated observed random vector
    $\mathbf{V}_o:=\Phi_\GG(\mathbf{B}',\mathbf{N})\in\MM(\GG)$, there does not exist
    another pair $ (\tilde{\mathbf{B}'},\Tilde{\mathbf{N}})\in\R^{\GG}\times\NGASS^{p}$ with 
    \begin{align}
        \label{eq:id:with:2}
        \Phi_{{\GG}}(\Tilde{\mathbf{B}'}, \Tilde{\mathbf{N}}) =_d \mathbf{V}_o  \ \text{ but } \  \phi(\mathbf{B}')\neq \phi(\tilde{\mathbf{B}}').
    \end{align}

    In latent variable models such as lvLiNGAM, global identifiability, as defined above, is often too stringent of a condition and fails to apply even in many commonly used models such as the instrumental variable model.  Thus, a weaker notion of so-called generic identifiability is  often considered for linear models \citep[\S16.4]{handbook}. We say that the parameter $\phi$ is \emph{generically} identifiable with (or without) knowledge of the graph, if the condition in  \eqref{eq:id:with:2} (or \eqref{eq:id:without:2}) holds for every $(\mathbf{B}',\mathbf{N})\in(\R^{\GG}\setminus \mathcal{B}_\text{except})\times\NGASS^{p}$ where $\mathcal{B}_\text{except}$ is a Lebesgue measure zero subset of $\R^\GG$. 
    
    Our results give necessary and sufficient conditions for generic identifiability of the parameters $[\mathbf{A}_{o,o}]_{i,j}$ and $[\mathbf{B}_{o}]_{i,j}$, both with and without knowledge of the underlying graph $\GG$.

\section{Main Results on Causal Effect Identification}
\label{sec:main:res}
This section presents our main identifiability results for lvLiNGAM. In Sections \ref{subsec:ceid:known:graph} and \ref{subsec:ceid:unknown:graph}, we characterize the identifiable causal effects, respectively, with or without knowledge of the graph. In Section \ref{subsec:examples}, we provide several examples in which our identifiability criteria hold; in \cref{subsec:cert}, we propose an algorithm to certify our criteria efficiently in time. Finally, in \cref{subsec:est:alg}, we present an estimation algorithm to estimate the causal effect from the observational data. 
 All proofs appear in the \cref{sec:proofs}.
 
\subsection{Identification with an Unknown Graph}
\label{subsec:ceid:unknown:graph}
In \citet{salehkaleybar:2020}, it is shown that if the total causal effect between any ancestor-descendent pair is non-zero, the mixing matrix $\mathbf{B}'$ can be identified by means of overcomplete independent component analysis (ICA), up to a permutation of the columns corresponding to a pair $(j,l)\in \mathcal{O}\times \mathcal{L}$ such that $\de_o(j)=\de_o(l)$. This leads to a graphical criterion for the identification of the entire mixing matrix, which we rephrase herein using our notation.

\begin{theorem}{\citep[Theo.~16]{salehkaleybar:2020}}
\label{thm:B:ident}    
    For  $\mathbf{B}'\in\R^\GG$,  matrix $\mathbf{B}_{o}$ is \emph{generically} identifiable without knowing the causal graph if and only if there are no pairs $j\neq i\in\mathcal{V}$ such that $\de_o(j)=\de_o(i)$.
\end{theorem}

\begin{remark}[The scaling matrix]
\label{rem:scaling}
    Equation \eqref{eq:b:can:matrix} implies that as long as we are focused on identifying the causal effect between observed variables alone, the scaling of the latent columns does not make a difference. Hence, without loss of generality, in the sequel, we assume that all the mixing matrices are scaled in such a way that the first non-zero entry in each column is equal to 1. In other words,  $\mathbf{A}_{il}=1$ if $i$ and $l$ are, respectively, observed and latent variables and $i$ is the first child of $l$ in a given causal order. Note that this  is always the case if $\de_o(l)=\de_o(i)$.
\end{remark}

In this section, we extend the result in Theorem \ref{thm:B:ident}  by providing \emph{necessary and sufficient} graphical conditions for the \emph{generic} identifiability for the entries of $\mathbf{B}_{o}$ and $\mathbf{A}_{o,o}$. 

Our first result refines the graphical condition of \cref{thm:B:ident} by adding \eqref{eq:tce:unknown:2}. This provides a complete graphical characterization of the identifiable total causal effects. 

\begin{theorem}[Total causal effect]
\label{thm:tot:eff}
    Consider any two observed variables  $i$ and $j$. The \emph{total causal effect} of $j$ on $i$ is \emph{generically} identifiable without knowledge of the graph \emph{if and only if} there is no $l\in\mathcal{L}$, such that    
    \begin{align}
    \label{eq:tce:unknown:1}
    \de_o(j) = \de_o&(l), \ \text{and}\\
    \label{eq:tce:unknown:2}
    i  \in\de^{\GG_{\setminus j}}_o&(l),
    \end{align}    
    where $\de^{\GG_{\setminus j}}_o(l)$ is the set of observed descendents of $l$ in the graph obtained from $\GG$ by removing all the edges pointing to $j$.
\end{theorem}

The next result provides a complete graphical characterization of the identifiable direct causal effects.

\begin{theorem}[Direct causal effect]
\label{thm:dce:id}
    Consider any two observed variables  $i$ and $j$. The direct causal effect of $j$ on $i$ is \emph{generically} identifiable without knowledge of the graph if and only if there are no pairs $(k, l)\in\mathcal{O}\times\mathcal{L}$ such that 
        \begin{align}
            \label{eq:dce:id:cond}
            &\de_o(k)=\de_o(l),\\
            \label{eq:dce:id:cond:1}
            &i\in\ch(l),\\
            \label{eq:dce:id:cond:2}
            &k\in\ch(j)\cup\{j\}.
        \end{align}
    
\end{theorem}

\begin{example}[Identification with instrumental variables]
\label{ex:iv:unknown:graph}

    The mixing matrix for the instrumental variable (IV) graph \citep[\S7.1]{cunningham:2021} in \cref{fig:IV:Unkwown:graph} has the following form
\begin{equation*}
    \mathbf{B}' = \begin{bmatrix}
        1 & 0 & 0 & 0\\
        b_{TI} & 1 & 1 & 0\\
        b_{TI}b_{YT} & b_{YT} & b_{YT}+b_{LY} & 1
    \end{bmatrix}.
\end{equation*}
The parameter of interest in the IV graph of  \cref{fig:IV:Unkwown:graph} is $b_{YT}$, i.e., the causal effect of the treatment $T$ on the outcome $Y$. Since $\de_o(T) =\de_o(L) = \{T, Y\}$, we can consider 
\begin{equation*}
    \begin{aligned}
        \Tilde{\mathbf{B}'} = \mathbf{B}'\cdot \mathbf{P}_{\sigma},\quad \Tilde{\mathbf{N}} = \mathbf{P}^{-1}_{\sigma}\cdot\mathbf{N}, 
    \end{aligned}
\end{equation*} where $\sigma$ is the transposition that permutes the columns corresponding to $T$ and $L$. From \eqref{eq:phi:map}, one can see that $$\Phi_{\tilde{\GG}_{IV}}(\Tilde{\mathbf{B}'}, \Tilde{\mathbf{N}}) =_d \Phi_{\GG_{IV}}({\mathbf{B}'}, \mathbf{N}), \qquad\forall\,\mathbf{N}\in\NGASS^p,$$ where $\Tilde{\GG}_{IV}$ is the graph obtained from the IV graph after adding an edge from $I$ to $Y$ (see \cref{appendix:fig:iv:graphs} in \cref{sec:proofs}). 

Since $[\Tilde{\mathbf{B}}']_{YT}\neq [{\mathbf{B}}']_{YT}$ the total causal effect of $T$ on $Y$ is not identifiable without knowing the true causal graph.

    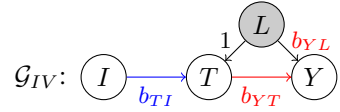
\begin{figure}[hb]
        \centering
        \begin{tikzpicture}[scale = 0.7]
        \node (I) at (-2.2,0) {$\GG_{IV}$:};
        \node[draw, circle, inner sep=2pt, minimum size=0.6cm] (I) at (-1,0) {$I$};
        \node[draw, circle, inner sep=2pt, minimum size=0.6cm] (T) at (1,0) {$T$};
        \node[draw, circle, inner sep=2pt, minimum size=0.6cm] (Y) at (3,0) {$Y$};
        \node[draw, circle, fill=gray!40, inner sep=2pt, minimum size=0.6cm] (H) at (2,1) {$L$};

        \draw[->, blue] (I) --node[below,blue] {\small$b_{TI}$} (T);
        \draw[->, red] (T) --node[below,red] {\small$b_{YT}$} (Y);
        \draw[->] (H) --node[left, pos = 0] {\small$1\,$} (T);
        \draw[->] (H) --node[right, pos = 0, red] {\small$\,b_{YL}$} (Y);

        \end{tikzpicture}
        \caption{Instrumental variable graph. The parameters in blue are identifiable without knowledge of the graph, while the parameters in red are not identifiable.}
        \label{fig:IV:Unkwown:graph}
    \end{figure}
Notice that we could derive this result by applying \cref{thm:dce:id} to the pair $(k, l) = (T, Y)$ directly.     
\end{example}

\subsection{Identification with a Known Graph}
\label{subsec:ceid:known:graph}
A permutation of two columns of the matrix $\mathbf{B}'$ may result in different graphs. This implies that if the graph is known, we can narrow down the set of possible permutations. In this section, we study how this additional assumption allows us to identify a parameter of interest (which is not generically identifiable). We will discuss an instance of this situation in \cref{ex:iv:known:graph}. 

The first result of this section provides a characterization of the column permutations that leave the graph unchanged.

\begin{theorem}
\label{thm:graph:id}
    For every $\mathbf{B}'$ outside a Lebesgue zero subset of $\R^\GG$, let $\Tilde{\mathbf{B}'} = \mathbf{B}'\cdot \mathbf{P}_{\sigma}$, where $\sigma$ is any permutation.    
    We have $\Tilde{\mathbf{B}'}\in\R^{\GG}$, if and only if 
    $$\de_o(i) = \de_o(\sigma(i)), \qquad\forall i\in\mathcal{V},$$ 
    and there are no $i, j\in \mathcal{O}$ and $l\in\mathcal{L}$ such that 
    \begin{equation*}
        i \in\pa(j)\cup\{j\}, \quad \sigma(j) = l, \quad \ch(l)\setminus\ch(i) \neq \emptyset. 
    \end{equation*}
\end{theorem}

\begin{example}[\cref{ex:iv:unknown:graph} continued]
\label{ex:iv:known:graph}
 We now show that the causal effect of $T$ on $Y$ is identifiable if we assume the graph in \cref{fig:IV:Unkwown:graph} is the true underlying graph. 
 
 The adjacency matrix corresponding to $\tilde{\mathbf{B}'}$ has the following form
\begin{equation*}
    \Tilde{\mathbf{A}}_{o,o} = \mathbf{I}-\tilde{\mathbf{B}}_{o}^{-1} = \begin{bmatrix}
        1 & 0 & 0 \\
        b_{TI} & 1 & 0 \\
        -b_{TI}b_{YL} & b_{YT}+b_{YL} & 1
    \end{bmatrix}.
\end{equation*}
This form is not compatible with the graph in \cref{fig:IV:Unkwown:graph} since the entry corresponding to the edge from $I$ to $Y$ is nonzero. Hence, if we were told that the graph in \cref{fig:IV:Unkwown:graph} is the true underlying causal structure, the only valid permutation of $\mathbf{B}'$ would be the identity. Therefore, we can identify  $b_{YT}$, the parameter of our interest.

Notice that we could derive this result by applying \cref{thm:graph:id} to the triple $(i, j, l) = (I, T, L)$ directly. 
See \cref{app:example:iv} for more discussion of the IV graph.
\end{example}

Using \cref{thm:graph:id}, we can refine the criteria of \cref{subsec:ceid:unknown:graph} for the setting of known causal graph. In particular, the next two theorems characterize all the total and direct causal effects that are identifiable when the graph is known.

\begin{theorem}[Total Causal Effect]
\label{thm:tot:eff:known:graph}
    Consider any two observed variables  $i$ and $j$. The \emph{total causal effect} of $j$ on $i$ is \emph{generically} identifiable with knowledge of the graph \emph{if and only if} there is no $l\in\mathcal{L}$, such that
    \begin{align}
        \label{eq:tot:ce:id:known:1}
        \de_o(j)=\de_o(l),&\\
        \label{eq:tot:ce:id:known:2}
        i\in\de^{\GG_{\setminus j}}_o(l),&\\
        \label{eq:tot:ce:id:known:3}
        \ch(l)\setminus\ch(k)=\emptyset,&\qquad\forall k\in\pa(j)\cup\{j\}.
    \end{align}
\end{theorem}

\begin{theorem}[Direct Causal Effect]
\label{thm:dce:known:graph}
    Consider any two observed variables  $i$ and $j$. The direct causal effect of $j$ on $i$ is \emph{generically} identifiable if and only if there are no pairs $(k,l)\in\mathcal{O}\times\mathcal{L}$ such that    
        \begin{align}
            \label{eq:dce:cond:2}
            \de_o(k)=\de_o(l),&\\
            \label{eq:dce:cond:3}
            i\in\ch(l),&\\
            \label{eq:dce:cond}
            k\in\ch(j)\cup\{j\},&\\                        
            \label{eq:dce:cond:1}
            \ch(l)\setminus\ch(k_1)=\emptyset,&\qquad\forall k_1\in\pa(k)\cup\{k\}.
        \end{align}
\end{theorem}

We conclude the section by giving the graphical condition for identification of the whole mixing matrix with knowledge of the graph.

\begin{corollary}[Mixing matrix identification]
\label{thm:mixing:matrix:known}
    For  $\mathbf{B}'\in\R^\GG$,  the entire matrix $\mathbf{B}_{o}$ is \emph{generically} identifiable with knowledge of the graph if and only if there are no $i, j\in~\mathcal{O}$, and $l\in\mathcal{L}$ such that \eqref{eq:tot:ce:id:known:1}, \eqref{eq:tot:ce:id:known:2}, and \eqref{eq:tot:ce:id:known:3} are satisfied.
\end{corollary}

\subsection{Examples}
\label{subsec:examples}
We now highlight different scenarios in which the results in this section allow us to relax standard assumptions in identifying the causal effect. For the sake of simplicity in presentation, in the first two scenarios, we only consider two latent confounders in the system. Although these scenarios can be easily extended to any arbitrary number of latent confounders.
\begingroup
\setlength{\itemindent}{-10pt}
\setlength{\leftmargini}{0pt}
\begin{enumerate}
    \item \textbf{Proxy Variables}\\
    The presence of proxy variables allow the identification of the causal effect in linear models. In particular, it has been shown that the causal effect of treatment $T$ on outcome $Y$ is identifiable if the following two conditions hold (see, e.g., \citet[\S4]{kuroki:2014}, \citet[\S2]{liu:2023}):
    \begin{enumerate}
        \item There are as many proxies as there are confounders,
        \item $W\indep (T, Y)\mid L$ for every proxy variable $W$ and every latent confounder $L$.
    \end{enumerate}
 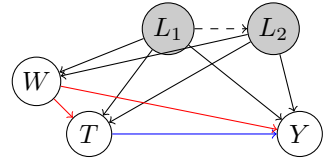
\begin{figure}[ht]
        \centering
        \begin{tikzpicture}[scale = 0.7]
            \node[draw, circle, fill=gray!40, inner sep=2pt, minimum size=0.6cm] (z1) at (1,0) {$L_1$};
            \node[draw, circle, fill=gray!40, inner sep=2pt, minimum size=0.6cm] (z2) at (3,0) {$L_2$};
            
            \node[draw, circle, inner sep=2pt, minimum size=0.6cm] (T) at (-0.5,-2) {$T$};
            \node[draw, circle, inner sep=2pt, minimum size=0.6cm] (Y) at (3.5,-2) {$Y$};
            
            \node[draw, circle, inner sep=2pt, minimum size=0.6cm] (W) at (-1.5,-1) {$W$};
            
            \draw[black, ->, dashed] (z1) to (z2);
            \draw[black, ->] (z1) to (T);
            \draw[black, ->] (z1) to (Y);
            \draw[black, ->] (z1) to (W);
            
            \draw[black, ->] (z2) to (T);
            \draw[black, ->] (z2) to (Y);
            \draw[black, ->] (z2) to (W);

            \draw[blue, ->] (T) to (Y);
            \draw[red, ->] (W) to (T);
            \draw[red, ->] (W) to (Y);
\end{tikzpicture}
        \caption{Causal graph with  proxy variable.         
        The parameter corresponding to the blue edge is generically identifiable if the graph is known, while those  corresponding to the red edges are not. The dashed edge is dropped in the corresponding canonical model.}
        \label{fig:proxy}
    \end{figure}
    
   Consider the causal graph in \cref{fig:proxy}. Both aforementioned conditions are violated as there are two latent confounders but a single proxy and there is an edge from this proxy to the treatment.
   Yet, in lvLiNGAM, we show that the causal effect from $T$ to $Y$ is identifiable.   
   For both the latent variables we have $W\in\de_o(L)\setminus\de_o(T)$, hence condition (\ref{eq:tot:ce:id:known:1}) cannot be satisfied. Using \cref{thm:tot:eff:known:graph}, this implies that the causal effect of our interest is identifiable.
   
    \item \textbf{Longitudinal Data}\\    
    The causal graph associated with a longitudinal data model \citep{imai:2017} is given in the left plot in \cref{fig:panel}, where $L$ is the latent confounder, $T$ is the time-varying treatment, and $Y$ is the time-varying outcome. The common identifiability assumptions in a linear setting are that the causal effect is constant through time, i.e., $[\mA]_{Y_1, T_1}=[\mA]_{Y_2, T_2}$ and there is no time-varying confounding, see, e.g., \citet[\S8]{cunningham:2021}. In  lvLiNGAM, these assumptions can be relaxed by having access to some covariates. In particular, suppose that there is a covariate $C_i$ such that $\de_o(L_i) = \de_o(C_i)$ for every time period $i$ (see the right plot in \cref{fig:panel}). We have $C_i\in\de_o(L_i)\setminus\de_o(T_i)$ and the condition (\ref{eq:tot:ce:id:known:1}) is not satisfied. This implies that the causal effect of $T_i$ on $Y_i$ is identifiable for every time period $i$.
        
        \begin{figure}[ht]
        \centering
        \begin{tikzpicture}[scale = 0.7]
            \node[draw, circle, fill=gray!40,inner sep=2pt, minimum size=0.6cm] (UU) at (4,0) {$L$};

            \node[draw, circle, inner sep=2pt, minimum size=0.6cm] (CC) at (4,-1) {$C$};
            
            \node[draw, circle, inner sep=2pt, minimum size=0.6cm] (TT1) at (3,-2) {$T_1$};
            \node[draw, circle, inner sep=2pt, minimum size=0.6cm] (TT2) at (5,-2) {$T_2$};

            \node[draw, circle, inner sep=2pt, minimum size=0.6cm] (YY1) at (3,-4) {$Y_1$};
            \node[draw, circle, inner sep=2pt, minimum size=0.6cm] (YY2) at (5,-4) {$Y_2$};

            \draw[black, ->] (UU) to (CC);
            
            \draw[black, ->] (UU) to (TT1);
            \draw[black, ->] (UU) to (TT2);
        
            \draw[black, ->] (UU) to (YY1);
            \draw[black, ->] (UU) to (YY2);

            \draw[red, ->] (CC) to (TT1);
            \draw[red, ->] (CC) to (TT2);
        
            \draw[red, ->] (CC) to (YY1);
            \draw[red, ->] (CC) to (YY2);
            
            \draw[blue, ->] (TT1) to (TT2);
            \draw[blue, ->] (TT1) to (YY1);
            \draw[blue, ->] (TT2) to (YY2);
            
            \node[draw, circle, fill=gray!40, inner sep=2pt, minimum size=0.6cm] (u1) at (9,0) {$L_1$};
            \node[draw, circle, fill=gray!40, inner sep=2pt, minimum size=0.6cm] (u2) at (11,0) {$L_2$};

            \node[draw, circle, inner sep=2pt, minimum size=0.6cm] (c1) at (8,-1) {$C_1$};
            \node[draw, circle, inner sep=2pt, minimum size=0.6cm] (c2) at (10,-1) {$C_2$};
            
            \node[draw, circle, inner sep=2pt, minimum size=0.6cm] (x1) at (7,-2) {$T_1$};
            \node[draw, circle, inner sep=2pt, minimum size=0.6cm] (x2) at (9,-2) {$T_2$};

            \node[draw, circle, inner sep=2pt, minimum size=0.6cm] (y1) at (7,-4) {$Y_1$};
            \node[draw, circle, inner sep=2pt, minimum size=0.6cm] (y2) at (9,-4) {$Y_2$};
            
            \draw[black, ->, dashed] (u1) to (u2);
            \draw[black, ->] (u1) to (c1);
            \draw[black, ->] (u1) to (c2);
            \draw[black, ->] (u2) to (c2);
            \draw[red, ->] (c1) to (c2);
            
            \draw[black, ->] (u1) to[bend right = 30] (x1);
            \draw[black, ->] (u1) to (x2);
            \draw[black, ->] (u2) to[bend right =30] (x2);

            \draw[black, ->] (u1) to (y1);
            \draw[black, ->] (u1) to[bend left = 20] (y2);
            \draw[black, ->] (u2) to (y2);

            \draw[red, ->] (c1) to (x1);
            \draw[red, ->] (c1) to (x2);
            \draw[red, ->] (c2) to (x2);

            \draw[red, ->] (c1) to (y1);
            \draw[red, ->] (c1) to (y2);
            \draw[red, ->] (c2) to (y2);

            \draw[blue, ->] (x1) to (x2);
            \draw[blue, ->] (x1) to (y1);
            \draw[blue, ->] (x2) to (y2);
            
        \end{tikzpicture}
        \caption{The causal graphs for longitudinal data. The graph on the right allows for a time-varying confounder. The parameters corresponding to the blue edges are generically identifiable with the knowledge of the graph, while the ones corresponding to the red edges are not. The dashed line is dropped in the corresponding canonical model.}
        \label{fig:panel}
    \end{figure}
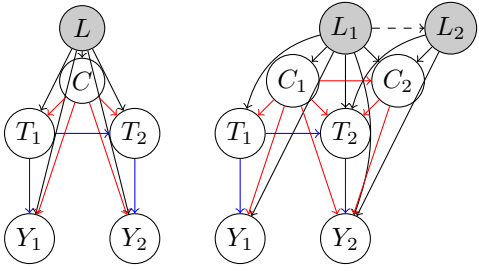
    \item\textbf{Underspecified Instruments}\\
    A standard assumption for identifying the causal effect using instrumental variables is that there are at least as many instruments as treatments \citep[Thm.~1]{brito:2002}. Recently, \citet{ailer:23} showed that in the linear underspecified setting, i.e., when the number of instruments is less than the number of treatments, one can identify the projection of the treatment on the instrument space, but this can be different from the causal effect.
    We now show that in the lvLiNGAM model, the treatment effects are identifiable in the underspecified case.
    
    We say that $I$ is a valid instrument for the treatments $T_1, \dots, T_n$ on $Y$ if the following conditions hold
    \begin{align*}
        I& \in\pa(T_i)\quad\forall\, i\in \{1,\dots,n\},\\
        I&\dsep_{\GG_{\setminus T}} Y,
    \end{align*}
    where $\dsep$ denotes d-separation \citep[\S1.2]{pearl:2009}, and $\GG_{\setminus T}$ is the graph obtained from $\GG$ by removing all the edges from $T_i$ to $Y$. See \cref{fig:IV} for an instance with two treatments and one instrument.

    For every latent variable $L_i$ such that $\de_o(T_i) = \de_o(L_i)$ and $Y\in\ch(L_i)$, we have $Y\in\ch(L_i)\setminus\ch(I)$. Hence condition \eqref{eq:tot:ce:id:known:3} does not hold. By \cref{thm:tot:eff:known:graph}, in lvLiNGAM, the causal effect of $T_i$ on $Y$ is identifiable for every $i$, even when only one instrumental variable is available. 
    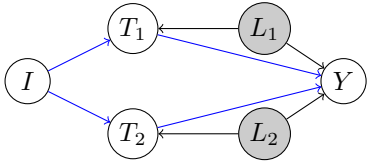
\begin{figure}[ht]
        \centering
        \begin{tikzpicture}[scale = 0.7]

        \node[draw, circle, inner sep=2pt, minimum size=0.6cm] (I) at (-1,0) {$I$};
        
        \node[draw, circle, inner sep=2pt, minimum size=0.6cm] (T1) at (1,1) {$T_1$};
        \node[draw, circle, inner sep=2pt, minimum size=0.6cm] (T2) at (1,-1) {$T_2$};

        \node[draw, circle, inner sep=2pt, minimum size=0.6cm] (Y) at (5,0) {$Y$};
        
        \node[draw, circle, fill=gray!40, inner sep=2pt, minimum size=0.6cm] (H1) at (3.5,1) {$L_1$};
        \node[draw, circle, fill=gray!40, inner sep=2pt, minimum size=0.6cm] (H2) at (3.5,-1) {$L_2$};

        \draw[blue,  ->] (I) to (T1);
        \draw[blue, ->] (I) to (T2);
        
        \draw[blue, ->] (T1) to (Y);
        \draw[blue, ->] (T2) to (Y);
        
        \draw[black, ->] (H1) to (T1);
        \draw[black, ->] (H1) to (Y);

        \draw[black, ->] (H2) to (T2);
        \draw[black, ->] (H2) to (Y);
        
        \end{tikzpicture}
        \caption{An example of causal graph for underspecified instrumental variable. The parameters corresponding to the blue edges are generically identifiable with knowledge of the graph.}
        \label{fig:IV}
    \end{figure}
    
    \end{enumerate}
\endgroup
\subsection{Certifying Identifiability}
\label{subsec:cert}
In this section, we prove that the graphical condition for identification given in \cref{thm:tot:eff:known:graph} can be certified in polynomial time in the size of the graph, and we provide a sound and complete algorithm (Algorithm \ref{alg:tce:known}) for this task.

The algorithms for the other identification criteria can be found in \cref{subsec:proof:cert}.

\begin{theorem}
\label{thm:cert}
    \cref{alg:tce:known} is sound and complete for certifying the generic identifiability of the total causal effect of $j$ on $i$ with knowledge of the graph $\GG$.
    The computational complexity of the algorithm is $\mathcal{O}(p_l(p_o^2+|E|))=\mathcal{O}(p^3)$.
\end{theorem}

\begin{algorithm}
    \caption{Total Causal Effect Identification with Known Graph}
    \label{alg:tce:known}
    \textbf{INPUT:} $\mathcal{V}=\mathcal{O}\cup\mathcal{L}, \GG, \{\ch(i) \mid i \in \mathcal{V}\} ,(j,i)$\\
    \begin{algorithmic}[1]
    \STATE $\text{ID} \gets \text{TRUE}$ 
    \STATE Sort $\mathcal{V}$ according to an ascending topological order
    \STATE Compute $\de_o(j)$    
    \WHILE{$\text{ID} == \text{TRUE}$ \AND $|\mathcal{L}|>0$}
        \STATE $l\gets \mathcal{L}[1]$\COMMENT{The first element in the list}
        \STATE Sort $\ch(l)$ according to the topological order defined in step 2
        \IF{$\ch(l)[1] = j$} 
            \STATE Compute $\de^{\GG_{\setminus j}}_o(l)$
            \IF{$i\in\de^{\GG_{\setminus j}}_o(l)$}
                \STATE Compute $\de_o(l)$
                \IF[{\color{gray}\eqref{eq:tot:ce:id:known:2}}\color{black}]{$\de_o(l) =\de_o(j)$}                       
                    \STATE $\text{ID} \gets \text{FALSE}$ \COMMENT{{\color{gray}\eqref{eq:tot:ce:id:known:1}}}
                    \STATE Compute $\pa(j)$ 
                    \FORALL{$k \in \pa(j) \cup \{j\}$}
                        \IF{$\ch(l) \setminus \ch(k) \neq \emptyset$}
                            \STATE $\text{ID} \gets \text{TRUE}$ \COMMENT{{\color{gray}\eqref{eq:tot:ce:id:known:3}}}
                        \ENDIF 
                    \ENDFOR
                \ENDIF
            \ENDIF
        \ENDIF
        \STATE $\mathcal{L}\gets \mathcal{L}\setminus\{l\}$
    \ENDWHILE
    \STATE \textbf{RETURN:} ID
    \end{algorithmic}
\end{algorithm}

\begin{remark}
\label{rem:id:algorithm}
    \cref{alg:tce:known} is simple in the sense that it
    directly checks the graphical conditions in \cref{thm:tot:eff:known:graph}. This is not the case for checking most identifiability results in linear models, which often requires building an auxiliary graph and solving a maximum-flow problem on it (which becomes prohibitive for large graphs), \citep{brito:2006,kumor:2020,barber:2022}.    
\end{remark}

\subsection{Estimation Algorithms}
\label{subsec:est:alg}
When the graph structure is unknown, \citet[Alg.~1]{salehkaleybar:2020} proposed an algorithm that first solves an overcomplete ICA problem and then post-process the estimated mixing matrix to enumerate all the possible causal effects. This usually entails solving a high-dimensional non-convex optimization problem. If the DAG structure is known, one can enforce this knowledge to reduce the problem's dimensionality from $p^2$ to $|E|$ and improve the performance. We follow this approach and propose an adaptation of the 
RICA algorithm for recovering the mixing matrix \citep{le:2011}. 

The objective function optimized by  RICA is a weighted sum of two terms; the first term is a contrast function that measures the non-Gaussianity of the exogenous noise, e.g., the $l_1$-loss, and the second term is a reconstruction loss that enforces the orthonormality of the rows of the mixing matrix. The only instance in which the rows of a matrix in $\R^{\GG}$ might be orthonormal is when all the causal effects are zero. Hence, we drop the reconstruction loss and only optimize the contrast function.

For a given DAG $\GG$, given observed data $\mathbf{X}_1,\dots,\mathbf{X}_N\in\R^{p_o}$, and a contrast function $g$, our algorithm solves the following optimization problem
\begin{equation*}
\textbf{Graphical RICA:}\argmin_{\mathbf{B}'\in\R^\GG}\frac{1}{N}\sum_{i=1}^Ng((\mathbf{B}')^T\cdot\mathbf{X}_i).
\end{equation*}
We evaluate the performance of our algorithm in comparison with existing methods in \cref{subsec:est:exp}.

\section{Related Work}
\label{sec:rel:work}
There is a rich literature on graphical criteria for the identifiability of causal effects. In the nonparametric setting, the ID algorithm is a sound and complete algorithm that solves the global identification problem given the causal graph \cite{shpitser:2006, shpitser:2023}.

In the parametric case, most results are for the semi-Markovian and linear Gaussian models; for these models, a necessary and sufficient criterion for global identifiability is known, \cite{drton:2011}, while a complete characterization for generic identifiability remains unknown, \cite{kumor:2020}. For Gaussian models with explicit linear confounders, a sufficient graphical criterion for generic identifiability was proposed in \citet{barber:2022}. 

For linear non-Gaussian models, \citet{salehkaleybar:2020} proposed necessary and sufficient graphical criterion for the identifiability of the whole mixing matrix, \citet{yang:2022} defines a notion of equivalence class for lvLiNGAM models, and \citet{cai:2023} provide sufficient graphical conditions for the identification of the mixing matrix using explicit moment equations. \citet{kivva:2023} proposed an identification formula that works for the causal graph with one proxy variable; in contrast, \citet{shuai:2023} proved that if one assumes that only the treatment is non-Gaussian and pre-treatment covariates are available, then the causal effect can be identified in the presence of latent confounders. \citet{tramontano:2024} proposed a necessary and sufficient graphical criterion for the identifiability of the direct causal effect in linear models for acyclic-directed mixed graphs.
\section{Experimental Results\protect\footnote{The code to replicate the experiments can be found at\::\:\href{https://github.com/danieletramontano/Causal-Effect-Identification-in-LiNGAM-Models-with-Latent-Confounders}{https://github.com/danieletramontano/Causal-Effect-Identification-in-LiNGAM-Models-with-Latent-Confounders}.}}
\label{sec:exp:res}
\subsection{Identification}
\label{subsec:id:exp}
We used Algorithms \ref{alg:tce:known} and \ref{alg:tce:unknown} to check the identifiability of a causal effect for randomly selected edges in random graphs. The graphs are generated according to an Erd\H{o}s-Rényi model in which we ensure that the sampled graphs are canonical. The probability of the causal effect of a randomly selected edge being identifiable versus the probability of accepting an edge in the graph generation model is plotted in \cref{fig:id:exp}.
For each setup, we randomly sample $500$ graphs. 

Interestingly, we found that for all the graphs and all the edges we sampled, the corresponding causal effects are identifiable with the knowledge of the graph. As expected, when we do not assume the graph is known, the probability of identifying the causal effect of randomly selected edge drops and it depends both on the density of the graph and the proportion of observed nodes (see \cref{fig:id:exp}).

\begin{figure}[ht]
    \centering
    \includegraphics[scale = 0.27]{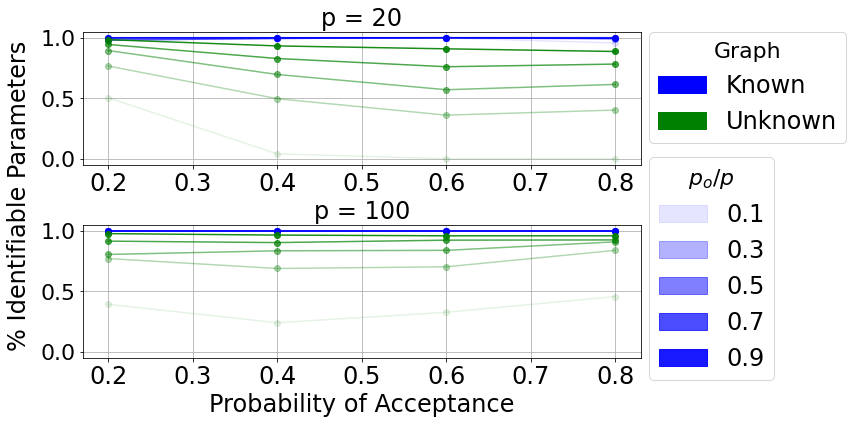}
    \caption{On the $x$-axis, the probability of acceptance of an edge. On the $y$-axis, the percentage of identifiable parameters.}
    \label{fig:id:exp}
\end{figure}
The average run time of \cref{alg:tce:known} for different graph sizes is shown in \cref{fig:running:time}. It is noteworthy that our algorithm can handle graphs with a thousand nodes in about a second; this is due to the simplicity of our identification criteria, as explained in \cref{rem:id:algorithm}.
\begin{figure}[ht]
    \centering
    \includegraphics[scale = 0.28]{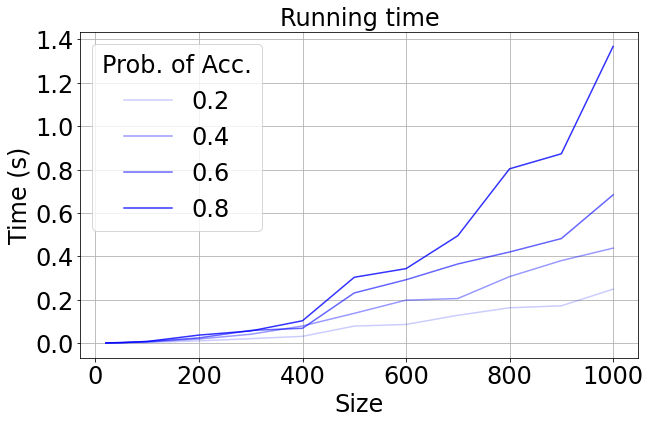}
\caption{On the $x$-axis, the size of the graph. On the $y$-axis, the average running time in seconds. $p_o/p$ is fixed to  0.5.}
    \label{fig:running:time}
\end{figure}

\subsection{Causal Effect Estimation}
\label{subsec:est:exp}
In this section, we provide experimental results for Graphical RICA (GRICA), which we introduced in \cref{subsec:est:alg}. 
We present experiments for data generated according to the causal structures in \cref{fig:PO settings}, and compare the performance of GRICA with the state-of-the-art. As a measure of performance, we used the relative error metric given by: $\textit{error} = |\text{Estimated Value} - \text{True Value}|/|\text{True Value}|$. Further experimental results are provided in \cref{app:exp}.
\begin{figure}
    \centering
    \begin{tikzpicture}[scale = 0.7]
        \node[draw, circle, fill=gray!40, inner sep=0pt, minimum size=0.5cm] (z1) at (1,0) {$L_1$};        
        
        \node[draw, circle, inner sep=0pt, minimum size=0.5cm] (T) at (0,-1) {$T$};
        \node[draw, circle, inner sep=0pt, minimum size=0.5cm] (Y) at (2,-1) {$Y$};        
        \node[draw, circle, inner sep=0pt, minimum size=0.5cm] (W) at (-0.5, 0) {$W$};
        \node at (1, -1.5) {$\GG_1$};
        
        \draw[black, ->] (z1) to (T);
        \draw[black, ->] (z1) to (Y);
        \draw[black, ->] (z1) to (W);
    
        \draw[blue, ->] (T) to (Y);

        \begin{scope}[xshift=3.5cm]
        \node[draw, circle, fill=gray!40, inner sep=0pt, minimum size=0.5cm] (z1) at (1,0) {$L_1$};
        
        \node[draw, circle, inner sep=0pt, minimum size=0.5cm] (T) at (0,-1) {$T$};
        \node[draw, circle, inner sep=0pt, minimum size=0.5cm] (Y) at (2,-1) {$Y$};
        \node[draw, circle, inner sep=0pt, minimum size=0.5cm] (W) at (-0.5, 0) {$W$};
        \node at (1, -1.5) {$\GG_2$};
        
        \draw[black, ->] (z1) to (T);
        \draw[black, ->] (z1) to (Y);
        \draw[black, ->] (z1) to (W);
        
        \draw[blue, ->] (T) to (Y);
        \draw[blue, ->] (W) to (T);
        \end{scope}

        \begin{scope}[xshift=7cm]
            \node[draw, circle, fill=gray!40, inner sep=0pt, minimum size=0.4cm] (z1) at (1,0) {$L_1$};
        
            \node[draw, circle, inner sep=0pt, minimum size=0.5cm] (z2) at (2.5, 0) {$Z$};
            
            \node[draw, circle, inner sep=0pt, minimum size=0.5cm] (T) at (0,-1) {$T$};
            \node[draw, circle, inner sep=0pt, minimum size=0.5cm] (Y) at (2,-1) {$Y$};
            
            \node[draw, circle, inner sep=0pt, minimum size=0.5cm] (W) at (-0.75,0) {$W$};
            \node at (1, -1.5) {$\GG_3$};
            
            \draw[black, ->] (z1) to (T);
            \draw[black, ->] (z1) to (Y);
            \draw[black, ->] (z1) to (W);
            \draw[black, ->] (z1) to (z2);
            \draw[blue, ->] (T) to (Y);
        \end{scope}
    \end{tikzpicture}
    \caption{The causal graphs considered in the experiments.}
    \label{fig:PO settings}
\end{figure}
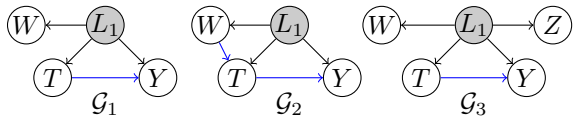

\textbf{1. Relative error vs sample size.}
\begin{figure}[t]
\centering
\begin{subfigure}[t]{0.5\textwidth}
    \centering
    \includegraphics[scale=0.4]{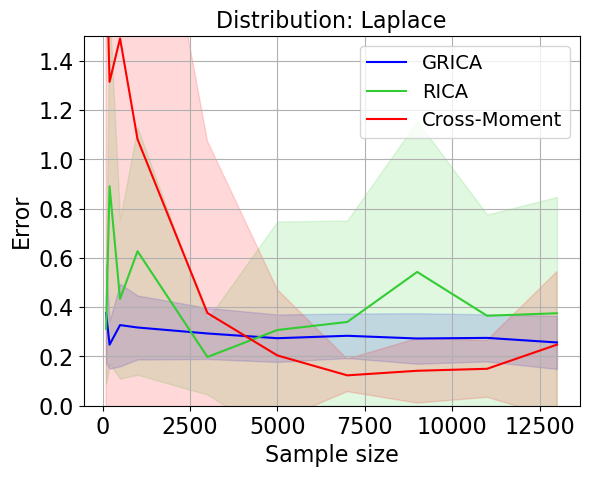}
    \caption{$\GG_1$.}
    \label{fig:Sample size vs error:g1}
\end{subfigure}

\begin{subfigure}[t]{0.5\textwidth}
    \centering
    \includegraphics[scale=0.4]{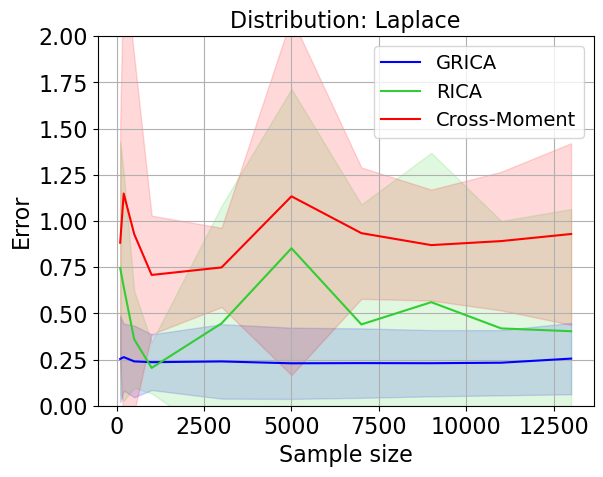}
    \caption{$\GG_2$.}
    \label{fig:Sample size vs error:g2}
\end{subfigure}
\caption{Relative error vs sample size}
\label{fig:Sample size vs error}
\end{figure}

We evaluated performance of the GRICA with respect to the observational sample size. We considered the settings that are compatible with the graphs $\GG_1$ and $\GG_2$ of \cref{fig:PO settings}
to recover the direct causal effect of variable $T$ on $Y$. We compared the performance of GRICA with the adaptation of the RICA algorithm implemented in \citet{salehkaleybar:2020} and the recently developed Cross-Moment method in \citet{kivva:2023}. 
We assumed all exogenous noises have the same distribution. Moreover, all direct causal coefficients in matrix $\mathbf{A}$ are generated uniformly at random from $[-1, -0.5]\cup[0.5, 1]$.\footnote{We consider this interval to ensure a fair comparison with RICA's implementation for causal effect estimation in \citet{salehkaleybar:2020} as it requires the absolute values of all causal coefficients to be smaller than one.} In \cref{fig:Sample size vs error}, we observe that GRICA consistently recovers the correct causal effect for both graphs, even with a few number of samples. Note that the Cross-Moment is a consistent estimator for the causal graph $\mathcal{G}_1$. It performs better than the rest when there are enough samples to compute high-order moments accurately. Furthermore, the RICA algorithm often gets stuck in bad local minima, and as a result is unstable.

\textbf{2. Relative error vs observations noise.} In \cref{fig:noise vs error}, we illustrate how the variances of certain exogenous noise impact the accuracy of the estimation. All causal coefficients in both settings are set to one. 

For the experiment over $\GG_1$, we scaled the standard deviation of exogenous noise $N_W$ corresponding to variable $W$ by a factor displayed on the $x$-axis of Figure \ref{fig:noise vs error:graph1}. We observed that GRICA performs similarly to the Cross-Moment method, which is specifically designed for the graph $\GG_1$.

The experiment for $\GG_3$ is similar to the one performed in \citet{kivva:2023}. Here, we scaled $N_W$ and $N_Z$ by $Ratio$ and $1/Ratio$, respectively, where $Ratio$ is plotted on the x-axis of Figure \ref{fig:noise vs error:graph3}. We compared the performance of GRICA with RICA and methods developed specifically for this causal graph in \cite{kivva:2023, tchetgen2020introduction}. As depicted in \cref{fig:noise vs error:graph3}, GRICA can benefit from the noiseless observation of $L_1$ through proxy $Z$, while all other algorithms are affected by the presence of noise in the observations from $W$.
\begin{figure}[t]
\centering
    \begin{subfigure}[t]{0.5\textwidth}
    \centering
        \includegraphics[scale=0.4]{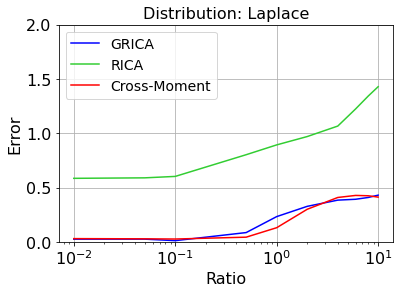}
        \caption{$\GG_1$.}
        \label{fig:noise vs error:graph1}
    \end{subfigure}

    \begin{subfigure}[b]{0.5\textwidth}
    \centering
        \includegraphics[scale=0.4]{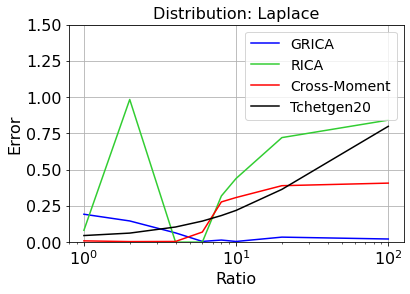}
        \caption{$\GG_3$.}
        \label{fig:noise vs error:graph3}
    \end{subfigure}
    \caption{Relative error vs noise scaling ratio.}
    \label{fig:noise vs error}
\end{figure}

\textbf{3. Causal effect estimation on random graphs.} In this section, we will present experiments for larger causal graphs. In particular, we compare the GRICA algorithm with the RICA algorithm on graphs randomly sampled from an Erd\H{o}s-Rényi model. The probability of accepting an edge is set to $1/2$. Similarly to the previous experiments, we generated all causal coefficients in the matrix $\mathbf{A}$ uniformly at random from $[-1, -0.5]\cup[0.5, 1]$. As a measure of performance, we use the normalized Frobenius loss between the estimated mixing matrix and the true one, i.e., $||\tilde{\mathbf{B}}-\mathbf{B}||_F/||\mathbf{B}||_F$.
Note that according to the experiments in \cref{subsec:id:exp}, we expect all the causal effects of interest to be identifiable. 
Figures \ref{fig:noise vs error:random graphs 1-5} and \ref{fig:noise vs error:random graphs 2-10} show the experimental results for the Erd\H{o}s-Rényi model in graphs with one latent and five observed variables and graphs with two latent and ten observed variables, respectively. The results are averaged over ten trials. As can be seen, GRICA  significantly outperforms the RICA algorithm in larger causal graphs.

\begin{figure}[htbp]
\centering
\begin{subfigure}[t]{0.5\textwidth}
\centering
    \includegraphics[scale=0.42]{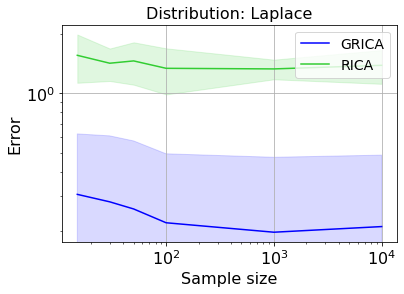}
    \caption{One latent and five observed variables.
    }
    \label{fig:noise vs error:random graphs 1-5}
\end{subfigure}

\begin{subfigure}[t]{0.5\textwidth}
\centering
    \includegraphics[scale=0.42]{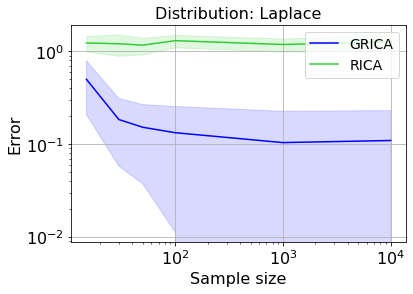}
    \caption{Two latent and ten observed variables.
    }
    \label{fig:noise vs error:random graphs 2-10}
\end{subfigure}
\caption{Results for the Erd\H{o}s-Rényi model.}
\label{fig:noise vs error:er}
\end{figure}

\subsection{Consistency Guarantees}
    The OICA problem is known to be identifiable but not separable; see \cite{eriksson:2004}. This implies that, as opposed to the complete ICA, minimizing a measure of non-Gaussianity does not necessarily lead to the identification of the mixing matrix. 
    
    In the past, various Expectation-Maximization (EM) algorithms have been proposed to work under specific parametric models for the exogenous noise model, e.g., Mixture of Gaussians \citep{olshausen:1999} or Laplace distribution \citep{lewicki:2000}. These algorithms have been proven to be consistent if the assumptions are satisfied. At the same time, they are very computationally demanding and have been shown to perform poorly in practice. For this reason, the core of the research on the topic has shifted to heuristic methods, in which the mixing matrix is found as a solution to a suitable smooth optimization problem. The RICA algorithm \citep{le:2011} is arguably the most prominent in this class of algorithms and has already been used in causality, see, e.g., \citet{yang:2022}. However, due to the complications stated above, the RICA algorithm has not yet been equipped with consistency guarantees.
    
    Being an adaptation of the RICA algorithm, also the large sample size performance of our algorithm is not well understood in rigorous mathematical terms. However, we point out that the same parametrization of the mixing matrix can be utilized in any OICA algorithm to reduce the dimensionality of the problem. What our experiments suggest is that this simple step can improve remarkably the performances of OICA algorithms, when applied to the estimation of causal effects.
\section{Conclusions}
\label{sec:conc}
We considered the problem of generic identifiability of causal effects in LiNGAM models when only observational data are available. We solved the problem by providing \emph{efficiently implementable}, graphical criteria for identification of the causal effect with and without knowledge of the graph. To estimate the effect, we proposed a flexible adaptation of the RICA algorithm, \citep{le:2011}, that incorporates the knowledge of the graph to reduce the dimension in the optimization problem.

To conclude, we highlight possible future directions.

    \textbf{\emph{Cyclic Models.}} When the acyclicity assumption is dropped, \citet{lacerda:2008} showed that, in the fully observed case, the structure of the graph can be recovered up to an equivalence class. There are no explicit graphical criteria for the identification of the causal effects in cyclic LiNGAM models. It is worth exploring whether our proof techniques could be extended to this setting.\\
    \textbf{\emph{Relaxing non-Gaussianity.}}  
    \citet{ng:2023} proved that in the complete ICA case, non-Gaussianity of the noises can be relaxed if the mixing matrix has a specific sparse structure. Extensions of these results to the overcomplete case might offer results on identification of the causal effects when some of the exogenous noises are allowed to be Gaussian.\\
    \textbf{\emph{Non-linear ICA.}} There is an active line of work that aims to exploit non-linear ICA, \citep{hyvarinen:2024}, to construct causal discovery algorithms, \citep{reizinger:2023}. There is no literature that leverages these results to understand the identifiability of the causal effects.

\section*{Acknowledgments}
This project has received funding from the European Research Council (ERC) under the European Union’s Horizon 2020 research and innovation programme (grant agreement No 883818) and supported in part by the SNF project 200021\_204355/1, Causal Reasoning Beyond Markov Equivalencies. DT's PhD scholarship is funded by the IGSSE/TUM-GS via a Technical University of Munich--Imperial College London Joint Academy of Doctoral Studies.

\section*{Impact Statement}
This paper presents work whose goal is to advance the field of 
Machine Learning. There are many potential societal consequences 
of our work, none which we feel must be specifically highlighted here.
\bibliography{ref}

\begin{thebibliography}{50}
\providecommand{\natexlab}[1]{#1}
\providecommand{\url}[1]{\texttt{#1}}
\expandafter\ifx\csname urlstyle\endcsname\relax
  \providecommand{\doi}[1]{doi: #1}\else
  \providecommand{\doi}{doi: \begingroup \urlstyle{rm}\Url}\fi

\bibitem[Adams et~al.(2021)Adams, Hansen, and Zhang]{adams:2021}
Adams, J., Hansen, N., and Zhang, K.
\newblock Identification of partially observed linear causal models: Graphical conditions for the non-gaussian and heterogeneous cases.
\newblock In \emph{Advances in Neural Information Processing Systems}, volume~34. Curran Associates, Inc., 2021.

\bibitem[Ailer et~al.(2023)Ailer, Hartford, and Kilbertus]{ailer:23}
Ailer, E., Hartford, J., and Kilbertus, N.
\newblock Sequential underspecified instrument selection for cause-effect estimation.
\newblock In \emph{Proceedings of the 40th International Conference on Machine Learning}, volume 202 of \emph{Proceedings of Machine Learning Research}. PMLR, 2023.

\bibitem[Athey \& Imbens(2017)Athey and Imbens]{athey:2017}
Athey, S. and Imbens, G.~W.
\newblock The state of applied econometrics: Causality and policy evaluation.
\newblock \emph{Journal of Economic Perspectives}, 31\penalty0 (2), 2017.

\bibitem[Axler(2015)]{sheldon:2014}
Axler, S.
\newblock \emph{Linear algebra done right}.
\newblock Springer, Cham, 3rd ed. edition, 2015.

\bibitem[Barber et~al.(2022)Barber, Drton, Sturma, and Weihs]{barber:2022}
Barber, R., Drton, M., Sturma, N., and Weihs, L.
\newblock Half-trek criterion for identifiability of latent variable models.
\newblock \emph{The Annals of Statistics}, 50, 2022.

\bibitem[Barrera \& Miljkovic(2022)Barrera and Miljkovic]{barrera:2022}
Barrera, E.~L. and Miljkovic, D.
\newblock The link between the two epidemics provides an opportunity to remedy obesity while dealing with covid-19.
\newblock \emph{Journal of Policy Modeling}, 44\penalty0 (2), 2022.

\bibitem[Brito \& Pearl(2002)Brito and Pearl]{brito:2002}
Brito, C. and Pearl, J.
\newblock Generalized instrumental variables.
\newblock In \emph{Proceedings of the Eighteenth Conference on Uncertainty in Artificial Intelligence}, UAI'02. Morgan Kaufmann Publishers Inc., 2002.

\bibitem[Brito \& Pearl(2006)Brito and Pearl]{brito:2006}
Brito, C. and Pearl, J.
\newblock Graphical condition for identification in recursive sem.
\newblock In \emph{Proceedings of the Twenty-Second Conference on Uncertainty in Artificial Intelligence}, UAI'06. AUAI Press, 2006.

\bibitem[Cai et~al.(2023)Cai, Huang, Chen, Hao, and Zhang]{cai:2023}
Cai, R., Huang, Z., Chen, W., Hao, Z., and Zhang, K.
\newblock Causal discovery with latent confounders based on higher-order cumulants.
\newblock In \emph{Proceedings of the 40th International Conference on Machine Learning}, ICML'23. JMLR.org, 2023.

\bibitem[Chiyohara et~al.(2023)Chiyohara, Furukawa, Noda, Morimoto, and Imamizu]{chiyohara:2023}
Chiyohara, S., Furukawa, J.-i., Noda, T., Morimoto, J., and Imamizu, H.
\newblock Proprioceptive short-term memory in passive motor learning.
\newblock \emph{Scientific Reports}, 13\penalty0 (1), 2023.

\bibitem[Ciarli et~al.(2023)Ciarli, Coad, and Moneta]{ciarli:2023}
Ciarli, T., Coad, A., and Moneta, A.
\newblock Does exporting cause productivity growth? evidence from chilean firms.
\newblock \emph{Structural Change and Economic Dynamics}, 66, 2023.

\bibitem[Cormen et~al.(2009)Cormen, Leiserson, Rivest, and Stein]{cormen:2009}
Cormen, T.~H., Leiserson, C.~E., Rivest, R.~L., and Stein, C.
\newblock \emph{Introduction to algorithms}.
\newblock MIT Press, Cambridge, MA, third edition, 2009.

\bibitem[Cox et~al.(2015)Cox, Little, and O'Shea]{cox:2015}
Cox, D.~A., Little, J., and O'Shea, D.
\newblock \emph{Ideals, varieties, and algorithms}.
\newblock Undergraduate Texts in Mathematics. Springer, Cham, fourth edition, 2015.
\newblock An introduction to computational algebraic geometry and commutative algebra.

\bibitem[Cunningham(2021)]{cunningham:2021}
Cunningham, S.
\newblock \emph{Causal inference: The mixtape}.
\newblock Yale university press, 2021.

\bibitem[de~Prado(2023)]{de:2023}
de~Prado, M. M.~L.
\newblock \emph{Causal Factor Investing: Can Factor Investing Become Scientific?}
\newblock Cambridge University Press, 2023.

\bibitem[Drton(2018)]{drton:2018}
Drton, M.
\newblock Algebraic problems in structural equation modeling.
\newblock In \emph{The 50th anniversary of {G}r\"{o}bner bases}, volume~77 of \emph{Adv. Stud. Pure Math.} Math. Soc. Japan, Tokyo, 2018.

\bibitem[Drton et~al.(2011)Drton, Foygel, and Sullivant]{drton:2011}
Drton, M., Foygel, R., and Sullivant, S.
\newblock {Global identifiability of linear structural equation models}.
\newblock \emph{The Annals of Statistics}, 39\penalty0 (2), 2011.

\bibitem[Eriksson \& Koivunen(2004)Eriksson and Koivunen]{eriksson:2004}
Eriksson, J. and Koivunen, V.
\newblock Identifiability, separability, and uniqueness of linear ica models.
\newblock \emph{Signal Processing Letters, IEEE}, 11, 2004.

\bibitem[Hoyer et~al.(2008)Hoyer, Shimizu, Kerminen, and Palviainen]{hoyer:2008}
Hoyer, P.~O., Shimizu, S., Kerminen, A.~J., and Palviainen, M.
\newblock Estimation of causal effects using linear non-gaussian causal models with hidden variables.
\newblock \emph{International Journal of Approximate Reasoning}, 49\penalty0 (2), 2008.
\newblock Special Section on Probabilistic Rough Sets and Special Section on PGM’06.

\bibitem[Hyv{\"a}rinen et~al.(2024)Hyv{\"a}rinen, Khemakhem, and Monti]{hyvarinen:2024}
Hyv{\"a}rinen, A., Khemakhem, I., and Monti, R.
\newblock Identifiability of latent-variable and structural-equation models: From linear to nonlinear.
\newblock \emph{Annals of the Institute of Statistical Mathematics}, 76, 2024.

\bibitem[Imai \& Kim(2019)Imai and Kim]{imai:2017}
Imai, K. and Kim, I.~S.
\newblock {When Should We Use Unit Fixed Effects Regression Models for Causal Inference with Longitudinal Data?}
\newblock \emph{American Journal of Political Science}, 63\penalty0 (2), 2019.

\bibitem[Kilbertus et~al.(2017)Kilbertus, Rojas~Carulla, Parascandolo, Hardt, Janzing, and Sch{\"o}lkopf]{kilbertus:2017}
Kilbertus, N., Rojas~Carulla, M., Parascandolo, G., Hardt, M., Janzing, D., and Sch{\"o}lkopf, B.
\newblock Avoiding discrimination through causal reasoning.
\newblock \emph{Advances in neural information processing systems}, 30, 2017.

\bibitem[Kivva et~al.(2023)Kivva, Salehkaleybar, and Kiyavash]{kivva:2023}
Kivva, Y., Salehkaleybar, S., and Kiyavash, N.
\newblock A cross-moment approach for causal effect estimation.
\newblock In \emph{Thirty-seventh Conference on Neural Information Processing Systems}, 2023.

\bibitem[Kumor et~al.(2020)Kumor, Cinelli, and Bareinboim]{kumor:2020}
Kumor, D., Cinelli, C., and Bareinboim, E.
\newblock Efficient identification in linear structural causal models with auxiliary cutsets.
\newblock In \emph{Proceedings of the 37th International Conference on Machine Learning}, volume 119 of \emph{Proceedings of Machine Learning Research}. PMLR, 2020.

\bibitem[Kuroki \& Pearl(2014)Kuroki and Pearl]{kuroki:2014}
Kuroki, M. and Pearl, J.
\newblock {Measurement bias and effect restoration in causal inference}.
\newblock \emph{Biometrika}, 101\penalty0 (2), 2014.

\bibitem[Lacerda et~al.(2008)Lacerda, Spirtes, Ramsey, and Hoyer]{lacerda:2008}
Lacerda, G., Spirtes, P., Ramsey, J., and Hoyer, P.~O.
\newblock Discovering cyclic causal models by independent components analysis.
\newblock In \emph{Proceedings of the Twenty-Fourth Conference on Uncertainty in Artificial Intelligence}, UAI'08. AUAI Press, 2008.

\bibitem[Le et~al.(2011)Le, Karpenko, Ngiam, and Ng]{le:2011}
Le, Q., Karpenko, A., Ngiam, J., and Ng, A.
\newblock Ica with reconstruction cost for efficient overcomplete feature learning.
\newblock In \emph{Advances in Neural Information Processing Systems}, volume~24. Curran Associates, Inc., 2011.

\bibitem[Lewicki \& Sejnowski(2000)Lewicki and Sejnowski]{lewicki:2000}
Lewicki, M. and Sejnowski, T.
\newblock Learning overcomplete representations.
\newblock \emph{Neural Computation}, 12:\penalty0 337--365, 02 2000.

\bibitem[Liu et~al.(2023)Liu, Tchetgen, and Varjão]{liu:2023}
Liu, J., Tchetgen, E. J.~T., and Varjão, C.
\newblock Proximal causal inference for synthetic control with surrogates.
\newblock \emph{arXiv:2308.09527}, 2023.

\bibitem[Maathuis et~al.(2019)Maathuis, Drton, Lauritzen, and Wainwright]{handbook}
Maathuis, M., Drton, M., Lauritzen, S., and Wainwright, M. (eds.).
\newblock \emph{Handbook of Graphical Models}.
\newblock Chapman \& Hall/CRC Handbooks of Modern Statistical Methods. CRC Press, Boca Raton, FL, 2019.

\bibitem[Micha{\l}ek \& Sturmfels(2021)Micha{\l}ek and Sturmfels]{michalek:2021}
Micha{\l}ek, M. and Sturmfels, B.
\newblock \emph{Invitation to nonlinear algebra}, volume 211 of \emph{Graduate Studies in Mathematics}.
\newblock American Mathematical Society, Providence, RI, 2021.

\bibitem[Michoel \& Zhang(2023)Michoel and Zhang]{michoel:2023}
Michoel, T. and Zhang, J.~D.
\newblock Causal inference in drug discovery and development.
\newblock \emph{Drug Discovery Today}, 28\penalty0 (10), 2023.

\bibitem[Ng et~al.(2023)Ng, Zheng, Dong, and Zhang]{ng:2023}
Ng, I., Zheng, Y., Dong, X., and Zhang, K.
\newblock On the identifiability of sparse ica without assuming non-gaussianity.
\newblock In \emph{Thirty-seventh Conference on Neural Information Processing Systems}, 2023.

\bibitem[Okamoto(1973)]{okamoto:1973}
Okamoto, M.
\newblock Distinctness of the eigenvalues of a quadratic form in a multivariate sample.
\newblock \emph{The Annals of Statistics}, 1\penalty0 (4), 1973.

\bibitem[Olshausen \& Millman(1999)Olshausen and Millman]{olshausen:1999}
Olshausen, B.~A. and Millman, K.~J.
\newblock Learning sparse codes with a mixture-of-gaussians prior.
\newblock In Solla, S.~A., Leen, T.~K., and M{\"{u}}ller, K. (eds.), \emph{Advances in Neural Information Processing Systems 12, {[NIPS} Conference, Denver, Colorado, USA, November 29 - December 4, 1999]}, pp.\  841--847. The {MIT} Press, 1999.

\bibitem[Pearl(2009)]{pearl:2009}
Pearl, J.
\newblock \emph{Causality}.
\newblock Cambridge University Press, Cambridge, second edition, 2009.
\newblock Models, reasoning, and inference.

\bibitem[Pe'er \& Hacohen(2011)Pe'er and Hacohen]{pe:2011}
Pe'er, D. and Hacohen, N.
\newblock Principles and strategies for developing network models in cancer.
\newblock \emph{Cell}, 144\penalty0 (6), 2011.

\bibitem[Reizinger et~al.(2023)Reizinger, Sharma, Bethge, Sch{\"o}lkopf, Husz{\'a}r, and Brendel]{reizinger:2023}
Reizinger, P., Sharma, Y., Bethge, M., Sch{\"o}lkopf, B., Husz{\'a}r, F., and Brendel, W.
\newblock Jacobian-based causal discovery with nonlinear {ICA}.
\newblock \emph{Transactions on Machine Learning Research}, 2023.

\bibitem[Salehkaleybar et~al.(2020)Salehkaleybar, Ghassami, Kiyavash, and Zhang]{salehkaleybar:2020}
Salehkaleybar, S., Ghassami, A., Kiyavash, N., and Zhang, K.
\newblock Learning linear non-gaussian causal models in the presence of latent variables.
\newblock \emph{Journal of Machine Learning Research}, 21\penalty0 (39), 2020.

\bibitem[Sanchez et~al.(2022)Sanchez, Voisey, Xia, Watson, O’Neil, and Tsaftaris]{sanchez:2022}
Sanchez, P., Voisey, J., Xia, T., Watson, H., O’Neil, A., and Tsaftaris, S.
\newblock Causal machine learning for healthcare and precision medicine.
\newblock \emph{Royal Society Open Science}, 9, 2022.

\bibitem[Shimizu(2022)]{shimizu:2022}
Shimizu, S.
\newblock \emph{Statistical Causal Discovery: LiNGAM Approach}.
\newblock Springer, 2022.

\bibitem[Shimizu et~al.(2006)Shimizu, Hoyer, Hyv\"{a}rinen, and Kerminen]{shimizu:2006}
Shimizu, S., Hoyer, P.~O., Hyv\"{a}rinen, A., and Kerminen, A.
\newblock A linear non-gaussian acyclic model for causal discovery.
\newblock \emph{Journal of Machine Learning Research}, 7, 2006.

\bibitem[Shpitser(2023)]{shpitser:2023}
Shpitser, I.
\newblock When does the id algorithm fail?
\newblock \emph{arXiv:2307.03750}, 2023.

\bibitem[Shpitser \& Pearl(2006)Shpitser and Pearl]{shpitser:2006}
Shpitser, I. and Pearl, J.
\newblock Identification of joint interventional distributions in recursive semi-markovian causal models.
\newblock In \emph{Proceedings of the 21st National Conference on Artificial Intelligence - Volume 2}, AAAI'06. AAAI Press, 2006.

\bibitem[Shuai et~al.(2023)Shuai, Luo, Zhang, Xie, and He]{shuai:2023}
Shuai, K., Luo, S., Zhang, Y., Xie, F., and He, Y.
\newblock Identification and estimation of causal effects using non-gaussianity and auxiliary covariates.
\newblock \emph{arXiv:2304.14895}, 2023.

\bibitem[Tchetgen et~al.(2020)Tchetgen, Ying, Cui, Shi, and Miao]{tchetgen2020introduction}
Tchetgen, E. J.~T., Ying, A., Cui, Y., Shi, X., and Miao, W.
\newblock An introduction to proximal causal learning.
\newblock \emph{arXiv preprint arXiv:2009.10982}, 2020.

\bibitem[Tramontano et~al.(2024)Tramontano, Drton, and Etesami]{tramontano:2024}
Tramontano, D., Drton, M., and Etesami, J.
\newblock Parameter identification in linear non-gaussian causal models under general confounding.
\newblock \emph{arXiv:2405.20856}, 2024.

\bibitem[Wang \& Drton(2023)Wang and Drton]{wang:cd:2023}
Wang, Y.~S. and Drton, M.
\newblock Causal discovery with unobserved confounding and non-gaussian data.
\newblock \emph{Journal of Machine Learning Research}, 24\penalty0 (271), 2023.

\bibitem[Wang et~al.(2023)Wang, Kolar, and Drton]{wang:2023}
Wang, Y.~S., Kolar, M., and Drton, M.
\newblock Confidence sets for causal orderings.
\newblock \emph{arXiv:2305.14506}, 2023.

\bibitem[Yang et~al.(2022)Yang, Ghassami, Nafea, Kiyavash, Zhang, and Shpitser]{yang:2022}
Yang, Y., Ghassami, A., Nafea, M., Kiyavash, N., Zhang, K., and Shpitser, I.
\newblock Causal discovery in linear latent variable models subject to measurement error.
\newblock \emph{Advances in Neural Information Processing Systems}, 35, 2022.

\end{thebibliography}
\bibliographystyle{icml2024}

\newpage
\appendix
\onecolumn
\label{appendix}
\section{Notions of Non-Linear Algebra}
In this section, we give the basic definitions of \emph{non-linear} algebra we will need for the proofs; we refer the interested reader to \citet{cox:2015,michalek:2021} for more details.
\begin{definition}
\label{def:ag:basics}
    For every natural number $n$, we denote the ring of polynomials in $n$ variables $x_1,\dots,x_n$ by $\R[x_1,\dots,x_n]$. Let $S$ be a, possibly infinite, subset of $\R[x_1,\dots,x_n]$. The affine variety associated to it is defined as  $\mathcal{V}(S)=\{x\in\R^n\mid f(x)=0,\,\forall f\in S\}$. The vanishing ideal associated to a variety $\mathcal{V}$ is $\mathcal{I}(\mathcal{V})=\{f\in\R[x_1,\dots,x_n]\mid f(x)=0\,\forall x\in\mathcal{V}\}$. The coordinate ring of $\mathcal{V}$ is defined as $\R[\mathcal{V}]=\R[x_1,\dots,x_n]/\mathcal{I}(\mathcal{V})$.
\end{definition}
\begin{lemma}[Lemma~\citep{okamoto:1973}]
\label{lem:okamoto}
Let $f(x_1,\dots,x_n)$ be a polynomial in real variables $x_1,\dots,x_n$, which is not identically zero. The set of zeros of the polynomial is a Lebesgue measure zero subset of $\R^n$.
\end{lemma}
\begin{remark}[Notation]
 For a given matrix $\mathbf{A}$, we denote the submatrix in which the $i$-th row and the $j$-th column are excluded by $[\mA]_{\setminus i,\setminus j}$.    
\end{remark}
\begin{lemma}
\label{lem:isomorphism}
    Let $\R^\GG_\mA$ and $\R^\GG$ defines as in \cref{subsec:model}. Then we have $\R^\GG\sim \R^\GG_\mA\sim \R^{|e|}$, where with the symbol $\sim$, we denote an isomorphism of affine varieties, see, e.g., \citet[Def.~6, \S5]{cox:2015} for a definition. Moreover $\R[\GG]$, $\R[\GG_\mA]$, and $\R[a_{i,j}\mid j\to i\in\GG]$ are isomorphic as rings.
\end{lemma}
\begin{proof}
    The isomorphism $\R^\GG_\mA\sim \R^{|e|}$ comes directly from its definition. Indeed it is easy to see that $\R^\GG_\mA$ is an $|e|$-dimensional linear subspace of $\R^{p\times p}=(a_{i,j})_{i,j\in p\times p}$, defined by the linear equations $a_{i, i}=1$, and $a_{i,j}=0$, $\forall i,j\in\mathcal{V}$ such that $j\to i\notin\GG$.

    To prove the isomorphism $\R^\GG\sim \R^\GG_\mA$, we need to prove that there is a polynomial bijective map between the two spaces. From \eqref{eq:b:can:matrix}, and using 
    $[\mathbf{B}_{o}]_{i,j}=[(\mathbf{A}_{o,o})^{-1}]_{i,j}=(-1)^{i+j}\det([\mathbf{A}_{o,o}]_{\setminus j,\setminus i}),$ where we used that $\det(\mathbf{A}_{o,o})=1$. It is clear that $\R^\GG$ is the image of polynomial map of $\R^\GG_{\mA}$. Let us call this polynomial map $\psi$ and assume $\psi(\mA)=\psi(\tmA)$. Then from the definition of $\psi$ we have $(I-\mA_{o,o})^{-1}=(I-\tmA_{o,o})^{-1}$ that implies $\mA_{o,o}=\tmA_{o,o}$. Moreover, $(I-\mA_{o,o})^{-1}\mA_{o,l}=(I-\tmA_{o,o})^{-1}\tmA_{o,l}$ that implies $\mA_{o,l}=\tmA_{o,l}$ and so $\mA=\tmA$.

    The isomorphisms between the rings come from \citet[\S5, Thm.~9]{cox:2015}.
\end{proof}

\begin{corollary}
    \label{cor:measure:zero}
    Let $f\in\R[\GG]$ be a non-zero polynomial. Then the subset of $\R^\GG$ on which $f$ vanishes is a Lebesgue measure 0 subset of $\R^\GG$.
\end{corollary}
\begin{proof}
    Thanks to the isomorphism in \cref{lem:isomorphism}, we can apply \cref{lem:okamoto} to $\R^\GG$.
\end{proof}
\begin{definition}[Leibniz Expansion of the Determinant]
    For any $M\in\R^{n\times n}$, the determinant of $M$ can be computed using the following formula:
    \begin{equation}
        \label{eq:leib:exp}
        \det(M)=\sum_{\sigma_P\in\mathcal{S}(n)}(-1)^{\sigma_P}\prod_{i=1}^nM_{\sigma_P(i),i},        
    \end{equation}
    where $\mathcal{S}(n)$ is the set of all the permutations of $n$ elements, and $(-1)^{\sigma_P}$ is the sign of the permutation. See, e.g., \citet[Def.~10.33]{sheldon:2014} for more details.
    \end{definition}

\begin{definition}
    Let $\pi\in\PP(j,i)$. The path monomial associated to it is defined as $$a^\pi=a_{i_1,i_2}\cdot\dots\cdot a_{i_{k},i_{k+1}}\in\R[\GG_{\mA}].$$
\end{definition}

    \begin{lemma}
    \label{lem:inv}
    Let $\mathbf{A}$ defined as in \eqref{eq:sem:1}. We have
        $$\mathbf{B}=(I-\mathbf{A})^{-1}=\displaystyle\sum_{i=0}^\infty\mathbf{A}^i=I+\mathbf{A}+\mathbf{A}^2+\dots+\mathbf{A}^p,$$
        $$[\mathbf{B}]_{i,j}=\sum_{P\in\mathcal{P}(i,j)}a^P.$$
        In particular $[\mathbf{B}]_{i,j}=0\in\R[\GG_\mA]$ if and only if $P\in\mathcal{P}(i,j)=\emptyset$.
    \end{lemma}

    \begin{definition}
        Let $I=\{i_1,\dots,i_n\},J=\{j_1,\dots,j_n\}\subset \mathcal{V}$. $P={P_1,\dots,P_n}$ is system of paths between $I,J$, if there exists a permutation $\sigma_P\in S_n$, such that $P_k\in\PP(i_k,j_{\sigma_P(k)})$. We denote the set of all such systems by $\PP(I, J)$. Moreover, a system of paths is non-intersecting if $P_k\cap P_l=\emptyset$ for $k\neq l$. The set of all such systems is denoted by $\tilde{\PP}(I, J)$. The path monomial associated to $P$, is defined as:
        $$ 
            a^P=a^{P_1}\cdot\dots\cdot a^{P_n}.
        $$
        \begin{figure}[ht]
        \centering
                \begin{tikzpicture}
           \node[] (1) at (0,0.5) {$i_1$};
           \node[] (2) at (0,-0.5)   {$i_2$};
           \node[] (3) at (1.5,0.5)   {$j_1$};
           \node[] (4) at (1.5,-0.5)  {$j_2$};
           
           \draw[->] (1) to (3);
           \draw[->] (2) to (4);
        \end{tikzpicture}
            \hspace{1cm}
        \begin{tikzpicture}
           \node[] (0) at (0,0) {$c$};
           \node[] (1) at (-1,0.5) {$i_1$};
           \node[] (2) at (-1,-0.5)   {$i_2$};
           \node[] (3) at (1,0.5)   {$j_1$};
           \node[] (4) at (1,-0.5)  {$j_2$};
           
           \draw[->] (1) to (0);
           \draw[->] (2) to (0);
           \draw[->] (0) to (3);
           \draw[->] (0) to (4);
        \end{tikzpicture}
        \caption{The system on the left has no sided intersection while the system on the right has.}
    \end{figure}
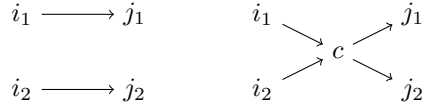
    \end{definition}
    
\begin{theorem}[Gessel-Viennot-Lindström Lemma]
Let $I, J\subset\mathcal{V}$ having the same size; then it holds that
\label{thm:gvl}
    $$\det(I-\mathbf{A})^{-1}_{I,J}=\sum_{P\in\mathcal{P}(J,I)}(-1)^{\sigma_P}a^P=\sum_{P\in\tilde{\mathcal{P}}(J,I)}(-1)^{\sigma_P}a^P.$$
In particular, $\det(I-\mathbf{A})^{-1}_{I,J}=0\in\R[\GG_\mA]$ if and only if $\tilde{\mathcal{P}}(J,I)=\emptyset$.
\end{theorem}

\section{Proofs and Lemmas}
\label{sec:proofs}
\subsection{Proofs of \cref{subsec:ceid:unknown:graph}}
\label{subsec:proofs:ceid:unknown:graph}
\begin{proof}[Proof of \cref{thm:tot:eff}]
From \citet[Theo.~15]{salehkaleybar:2020}, we know that the entire $j$-th column $\mathbf{B}'$ is identifiable if and only if there is no latent variable $l$ such that $\de_o(j)=\de_o(l)$. Thus, we know that if this condition is satisfied, the entry $[\mathbf{B}']_{i,j}$ is identifiable. Hence, we can assume that such an $l$ exists to conclude the proof. 

We will prove that $[\mathbf{P}_\sigma\mathbf{B}']_{i,j}=[\mathbf{B}']_{i,l}\neq [\mathbf{B}']_{i,j}$, \emph{in general}, if and only if $i\in\de^{\GG_{\setminus j}}_o(l)$, where $\sigma$ is the transposition that swaps $j$ and $l$. In particular, we will show that $[\mathbf{B}']_{i,j}-[\mathbf{B}']_{i,l}\in\R[\GG]$ where $\R[\GG]$ (the coordinate ring associate to $\GG$ as defined in \cref{def:ag:basics}), is a  \emph{non-zero} polynomial if and only if the graphical condition is satisfied. Notice that from \cref{cor:measure:zero}, we know this is enough to show that the parameters are generically identifiable.

Using \cref{lem:inv}, we can write the entries of $\mathbf{B}'$ as
$$[\mathbf{B}']_{i,l}=(I-\mathbf{A})^{-1}_{i,l}=\sum_{P\in\mathcal{P}(l,i)}a^P,$$
where $\mathcal{P}(l,i)$ denotes the set of all the paths from $l$ to $i$.
Now let $\mathcal{P}_{j}(l,i)$ be the set of directed path from $l$ to $i$ that passes through $j$, and $\mathcal{P}_{\setminus j}(l,i)=\mathcal{P}(l,i)\setminus\mathcal{P}_{j}(l,i)$. We can rewrite the formula above as follows
$$
[\mathbf{B}']_{i,l}=\sum_{P\in\mathcal{P}(l,i)}a^P=\sum_{P\in\mathcal{P}_j(l,i)}a^P+\sum_{Q\in\mathcal{P}_{\setminus j}(l,i)}a^Q.
$$
Note that $P\in\mathcal{P}_j(l,i)$ if and only if there are $P_j\in\mathcal{P}(l,j)$ and $P_i\in\mathcal{P}(j,i)$ such that $a^P=a^{P_j}a^{P_i}$. Finally, we can write $[\mathbf{B}']_{i,l}$ as:
\begin{equation}
\label{eq:b:ident}
    \begin{aligned}
        & \sum_{P\in\mathcal{P}_j(l,i)}a^P+\sum_{Q\in\mathcal{P}_{\setminus j}(l,i)}a^Q=\sum_{\substack{P_i\in\mathcal{P}(j,i)\\ P_j\in\mathcal{P}(l,j)}}a^{P_j}a^{P_i}+\sum_{Q\in\mathcal{P}_{\setminus j}(l,i)}a^Q\\
        =& \underbrace{\sum_{P_i\in\mathcal{P}(j,i)}a^{P_i}}_{[\mathbf{B}']_{i,j}}\underbrace{\sum_{P_j\in\mathcal{P}(l,j)}a^{P_j}}_{[\mathbf{B}']_{i,l}}+\sum_{Q\in\mathcal{P}_{\setminus j}(l,i)}a^Q
        = [\mathbf{B}']_{i,j}[\mathbf{B}']_{j,l}+\sum_{Q\in\mathcal{P}_{\setminus j}(l,i)}a^Q\\
        =&[\mathbf{B}']_{i,j}+\sum_{Q\in\mathcal{P}_{\setminus j}(l,i)}a^Q,
    \end{aligned}
\end{equation}
where we used $[\mB']_{j,l}=1$ from \cref{rem:scaling} that we can assume without loss of generality. Finally, we have $[\mathbf{B}']_{i_0,j}-[\mathbf{B}']_{i_0,l}=\sum_{Q\in\mathcal{P}_{\setminus j}(l,i_0)}a^Q\in\R[\GG_\mA]$. From \cref{lem:inv}, we know it is the zero polynomial if and only if $\mathcal{P}_{\setminus j}(l,i_0)=\emptyset$, i.e., $i_0\notin\de^{\GG_{\setminus j}}_o(l)$. Notice that we proved that $[\mathbf{B}']_{i,j}-[\mathbf{B}']_{i_0,l}$ is non zero in $\R[\GG_\mA]$, that is enough to show thanks to \cref{lem:isomorphism}.
\end{proof}

\begin{lemma}
\label{lem:adj:mat}
    Let $j$ and $l$ be observed and latent variables, respectively, such that $\de_o(j)=\de_o(l)$. Let $\tilde{\mathbf{B}}'=\mB'\mathbf{P}_\sigma$, where $\sigma$ is the transposition swapping the columns corresponding to $l$ and $j$.
    For every $i, k\in\mathcal{O}$, and $h\in\mathcal{L}$ we have:
    \begin{align}
    \label{eq:adj:mat}
        [\mathbf{\Tilde{A}}_{o,o}]_{i,k} &=
            [\mathbf{A}_{o,o}]_{i,k}+\mathbf{1}_{\ch(l)}(i)c^{j}_{l,i}[I-\mathbf{A}_{o,o}]_{j,k},\\
        \label{eq:adj:mat:lat}
        [\mathbf{\Tilde{A}}_{o,l}]_{i,h} &= [\mathbf{A}_{o,l}]_{i,h}-\mathbf{1}_{\ch(l)}(i)c^{j}_{l,i}[\mathbf{A}_{o,l}]_{j,h},
    \end{align}
    where $c^{j}_{l,i}=(-1)^{i+j}\sum_{Q\in\mathcal{P}_{\setminus j}(l,i)}a^Q$, and $\mathbf{1}_X$ represent the indicator function of a set $X$.
    
    In particular, $[\mathbf{\Tilde{A}}_{o,o}]_{i,k}=[\mathbf{A}_{o,o}]_{i,k}$, if $i\notin\ch(l)$ or $k\notin\pa(j)\cup\{j\}$.
\end{lemma}

\begin{proof}[Proof of \cref{lem:adj:mat}]
    We first show the equality in \eqref{eq:adj:mat}. According to \eqref{eq:b:ident}, $\Tilde{\mathbf{B}}_o$ can be written as $\mathbf{B}_o+C$:
    \begin{equation}
    \label{eq:Cij}
        C_{i,k}=
            \begin{cases}
                \displaystyle\sum_{Q\in\mathcal{P}_{\setminus j}(l,i)}a^Q         \quad&\emph{if }\,k = j\,\emph{ and }\, i\in \de^{\GG_{\setminus j}}_o(l),\\
                0          \quad&\emph{otherwise.}\\
            \end{cases}
    \end{equation}
    From the formula for the inverse of a matrix, we know that
    $$
        [\mathbf{B}_o+C]^{-1}_{i,k} = (-1)^{i+k}\det([\mathbf{B}_o+C]_{\setminus k,\setminus i}).
    $$
    Note that $\det(\mathbf{B}_o+C)=1$ since it is a lower triangular matrix with one on the diagonal. Thus, we can imply that $[\Tilde{\mathbf{A}}_{o,o}]_{j,k}= [(\mathbf{B}_o+C)^{-1}]_{j,k} =  [\mathbf{B}_o]^{-1}_{j,k} = [\mathbf{A}_{o,o}]_{j,k}$. Hence, we only need to consider the entries of $[\Tilde{\mathbf{A}}_{o,o}]_{i,k}$ where $i\neq j$.

   To compute $\det([\mathbf{B}_o+C]_{\setminus k,\setminus i})$ with respect to the $j$-th column, we have:
        \begin{align*}
            &\det([\mathbf{B}_o+C]_{\setminus k,\setminus i})=\displaystyle\sum(-1)^{s+j}[\mathbf{B}_o]_{s,j}\det([\mathbf{B}_o]_{\setminus(j,s),\setminus(i,j)})+\displaystyle\sum(-1)^{s+j}C_{s,j}\det([\mathbf{B}_o]_{\setminus(j,s),\setminus(i,j)})\\
            &=\det([\mathbf{B}_o]_{\setminus k,\setminus i})+\displaystyle\sum(-1)^{s+j}C_{s,j}\det([\mathbf{B}_o]_{\setminus(j,s),\setminus(i,j)})\\
            &=(-1)^{i+j}[\mathbf{A}_{o,o}]_{i,k}+\displaystyle\sum_{s\in \de^{\GG_{\setminus k}}_o(l)}(-1)^{s+j}\sum_{Q\in\mathcal{P}_{\setminus j}(l,s)}a^Q \det([\mathbf{B}_o]_{\setminus(k,s),\setminus(i,j)}),
        \end{align*}
        where in the last equality, we plug in $C_{s,j}$ according to \eqref{eq:Cij}.
        
        Based on Theorem \ref{thm:gvl}, 
        \begin{equation*}
        \begin{aligned}
            &\det([\mathbf{B}_o]_{\setminus(k,s),\setminus(i,j)})=\det([(I-\mathbf{A}_{o,o})^{-1}]_{\setminus(k,s),\setminus(i,j)})=\sum_{P\in\tilde{\mathcal{P}}(\mathcal{O}\setminus\{i,j\},\mathcal{O}\setminus\{k,s\})}(-1)^{\sigma_P}a^P.
        \end{aligned}
        \end{equation*}
        In the following, we show that that $\tilde{\mathcal{P}}(\mathcal{O}\setminus\{i,j\},\mathcal{O}\setminus\{k,s\})$ can be non-empty only if: 
        \begin{equation} 
        \label{eq:int:paths}
            \begin{cases}
                &s=i \emph{ or } s\to i\in\GG,\\
                &k=j \emph{ or } k\to j\in\GG.         
            \end{cases}
        \end{equation}
        To show this, let $$P=\{\pi_1,\dots,\pi_{j-1},\pi_{j+1},\dots,\pi_{i-1},\pi_{i+1},\dots,\pi_p\}$$ be a system of paths in $\tilde{\mathcal{P}}(\mathcal{O}\setminus\{i,j\},\mathcal{O}\setminus\{k,s\})$, where with $\pi_a$ we denote the path starting at the node $a$. Consider the path $\pi_s=sa_0\dots a_t$. If there is a $x\in\{0,\dots,t\}$ such that $a_x\notin\{i, j\}$, then $\pi_s$ would intersect $\pi_{a_x}$ and thus $P$ would have an intersection. The same argument applies for $\pi_k$. Therefore, 
        $$
            \pi_s=s a_0, \qquad \pi_k=k b_0, \qquad \{a_0, b_0\} = \{i, j\}.        
        $$
        To conclude the argument, please note that, since from \eqref{eq:Cij} we see that  $C_{s,j}=0$ if $s\notin\de^{\GG_{\setminus j}}_o(l)$, we can restrict ourselves to the case in which   $s\in\de^{\GG_{\setminus j}}_o(l)\subseteq\de(j)$. Since the graph is acyclic, this implies $a_0=i$ and $b_0 = j$.
        
        We now prove that if for all the paths from $l$ to $i$, there is an observed variable $s_0\neq i$ on the path, the quantity
        $$
            \displaystyle\sum_{s\in\de^{\GG_{\setminus k}}_o(l)}(-1)^{s+j}\sum_{Q\in\mathcal{P}_{\setminus j}(l,s)}a^Q \det([\mathbf{B}_o]_{\setminus(k,s),\setminus(i,j)}),
        $$
        is equal to 0.
        We can rewrite the sum above as follows:
        \begin{equation}
        \label{eq:path:can}
            \displaystyle\sum_{s\in\de^{\GG_{\setminus k}}_o(l)}\sum_{\substack{Q\in\mathcal{P}_{\setminus j}(l,s)\\P\in\tilde{\mathcal{P}}(\mathcal{O}\setminus\{i,j\},\mathcal{O}\setminus\{k,s\})}}(-1)^{s+j+\sigma_P}a^Q a^P.
        \end{equation}
        Let $\pi_0$ be a path from $l$ to $i$ that can be decomposed in the following way:
        $$
            \overbrace{\underbrace{l\to\dots\to s_0}_{Q_0}\to i}^{\pi_0},
        $$
        for some $s_0\in\mathcal{V}$. If $s_0$ is observed, the monomial associated with this path will appear in the sum twice in \eqref{eq:path:can}; the first time when considering $Q=Q_0$ and $s=s_0$ while the second time when considering $Q=\pi_0$ and $s=i$. The last thing to prove is that the sign will differ in the two cases. This comes from the Leibniz expansion for the determinant, \eqref{eq:leib:exp}. Indeed, the sign associated with each monomial is the sign associated with the permutation in the Leibniz formula. In the first case, the associated permutation is the following: 
        $$
            \sigma_1(a)=
            \begin{cases}
                a &\emph{if } a\notin\{s_0,i,j\},\\
                j &\emph{if } a=s_0,\\
                s_0 &\emph{if } a=i,\\
                k &\emph{if } a=j,
            \end{cases}
        $$
        while for the second case, the associated permutation is:
        $$
            \sigma_2(a)=
            \begin{cases}
                a &\emph{if } a\notin\{j,i\},\\
                j &\emph{if } a=i,\\
                k &\emph{if } a=j.                
            \end{cases}
        $$
        As the two permutations can be obtained one from the other via the transposition corresponding to swapping the columns corresponding to $j$ and $s_0$, they have different signs. 
        
        The same argument also implies that all the elements in \eqref{eq:path:can}, involving $s\in\pa(i)$ cancel out allowing us to rewrite the sum as:
        $$
            \sum_{\substack{Q\in\mathcal{P}_{\setminus j}(l,i)\\P\in\tilde{\mathcal{P}}(\mathcal{O}\setminus\{i,j\},\mathcal{O}\setminus\{i,k\})}}(-1)^{i+j+\sigma_P}a^Q a^P.
        $$
        With the same argument used to prove \eqref{eq:int:paths}, one can see that the only element in $\tilde{\mathcal{P}}(\mathcal{O}\setminus\{i,j\},\mathcal{O}\setminus\{i,k\})$, is the system of paths that sends $k$ to $j$ directly and all the other elements remain fixed. This system of paths has a negative sign if $k\neq j$ and a positive sign otherwise, implying that the sum is equal to:
        $$
            (-1)^{i+j}[I-\mathbf{A}_{o,o}]_{j,k}\sum_{Q\in\mathcal{P}_{\setminus j}(l,i)}a^Q ,
        $$
        which concludes the first part of the proof.
        We now prove \eqref{eq:adj:mat:lat}. Using \eqref{eq:b:can:matrix} we can write $[\mathbf{\Tilde{A}}_{o,l}]_{i,h}=[(\mathbf{\Tilde{B}}_{o,o})^{-1}(\mathbf{\Tilde{B}}_{o,l})]_{i,h}=\sum_{s\in \mathcal{O}}[(\mathbf{\Tilde{B}}_{o,o})^{-1}]_{i,s}[(\mathbf{\Tilde{B}}_{o,l})]_{s,l}$.
    
    We first prove the case $h=l$. In this case, by definition of $\Tilde{\mathbf{B}}'$, we have $[\mathbf{\Tilde{B}}_{o,l}]_{s,l}=[\mathbf{B}_{o}]_{s,j}$. Thus, plugging in \eqref{eq:adj:mat} we can write $[\mathbf{\Tilde{A}}_{o,l}]_{i,l}$ as
    \begin{equation*}
        \begin{aligned}
            &\sum_{s\in\mathcal{O}}[(\mathbf{B}_{o})^{-1}]_{i,s}[\mathbf{B}_{o}]_{s,j}-c^{j}_{l,i}\sum_{s\in\mathcal{O}}[(\mathbf{B}_{o})^{-1}]_{j,s}[\mathbf{B}_{o}]_{s,j}\\
            &=\delta_{i,j}-c^{j}_{l,i}=\delta_{i,j}-c^{j}_{l,i}[\mA]_{j,l},
        \end{aligned}
    \end{equation*}
where, again, we used $[\mA]_{j,l}=1$, that comes from Remark \ref{rem:scaling}.
    For the case $h\neq l$ we use again \eqref{eq:adj:mat}, and write $[\mathbf{\Tilde{A}}_{o,l}]_{i,h}$ as 
    \begin{equation*}
        \begin{aligned}
            &\sum_{s\in\mathcal{O}}[(\mathbf{B}_{o})^{-1}]_{i,s}[\mathbf{B}_{l}]_{s,h}-c^{j}_{l,i}\sum_{s\in\mathcal{O}}[(\mathbf{B}_{o})^{-1}]_{j,s}[\mathbf{B}_{l}]_{s,h}\\
            &=[\mathbf{A}_{o,l}]_{i,h}-c^{j}_{l,i}[\mathbf{A}_{o,l}]_{j,h},
        \end{aligned}
    \end{equation*}
    where for the last equality we only used $\mathbf{A}_{o,l}=(\mathbf{B}_{o})^{-1}\mathbf{B}_{l}$ that comes from \eqref{eq:b:can:matrix}.
\end{proof}

\begin{proof}[Proof of \cref{thm:dce:id}]
     From \cref{lem:adj:mat}, we can conclude that if there are no $k$ and $l$ satisfying the conditions in Equations \eqref{eq:dce:id:cond}-\eqref{eq:dce:id:cond:2}, then the entry $[\mathbf{A_{o,o}}]_{i,j}$ remains unchanged when swapping the columns of $\mathbf{B}'$, proving that the condition is sufficient. To prove the necessity, it is enough to show that $\sum_{Q\in\mathcal{P}_{\setminus j}(l, i)}a^Q\neq0\in\R[\GG_\mA]$, which is equivalent to proving that $\mathcal{P}_{\setminus j}(l, i)\neq\emptyset$. This is true since $i\in\ch(l)$ from \cref{eq:dce:id:cond:1}.
\end{proof}
\subsection{Proofs of \cref{subsec:ceid:known:graph}}
\label{subsec:proofs:ceid:known:graph}

\begin{lemma}
\label{lem:graph:id}
    For every $\mathbf{B}'$ outside of a Lebesgue zero subset of $\R^\GG$, let $\Tilde{\mathbf{B}'} = \mathbf{B}'\cdot \mathbf{P}_{\sigma}$, where $\sigma$ is the transposition that swaps $j$ and $l$, and $\mathbf{P}_{\sigma}$ is the associated permutation matrix.
    We have $\Tilde{\mathbf{B}'}\in\R^{\GG}$ if and only if  
    $$\de_o(I) = \de_o(\sigma(I)), \qquad\forall i\in\mathcal{V},$$ and there is no $k\in\pa(j)\cup\{j\}$, such that $\ch(l)\setminus\ch(k)\neq\emptyset$.
\end{lemma}

\begin{proof}[Proof of \cref{lem:graph:id}]
From \citet[Thm.~15]{salehkaleybar:2020}, we know that the only permutations that result in matrices in $\R^{\Tilde{\GG}}$ for some DAG $\Tilde{\GG}$, are the ones for which $$\de_o(I) = \de_o(\sigma(I)), \qquad\forall i\in\mathcal{V}.$$ Hence this condition is necessary.

The edges in $\tilde{\GG}$ are given by the support of the matrix $\Tilde{\mathbf{A}}$. In particular $\GG=\tilde{\GG}$ if and only if 
\begin{equation}
    \begin{aligned}
        \label{eq:graph:eq:cond}[\mathbf{A}_{o,o}]_{i,k}=0\in\R[\GG_\mA]&\iff[\Tilde{\mathbf{A}}_{o,o}]_{i,k}=0\in\R[\GG_\mA]\\
        [\mathbf{A}_{o,l}]_{i,l}=0\in\R[\GG_\mA]&\iff[\Tilde{\mathbf{A}}_{o,l}]_{i,l}=0\in\R[\GG_\mA]
    \end{aligned}    
\end{equation}
We provide the proof only for $\mathbf{A}_{o,o}$ since the one for $\mathbf{A}_{o,l}$ follows the same argument only using \eqref{eq:adj:mat:lat} instead of \eqref{eq:adj:mat}. 

From \cref{lem:adj:mat} we know that $[\mathbf{\Tilde{A}}_{o,o}]_{i,k}=[\mathbf{A}_{o,o}]_{i,k}$, if $i\notin\ch(l)$ or $k\notin\pa(j)\cup\{j\}$. Thus, the only the entries to consider are $i\in\ch(l)$ and $k\in\pa(j)\cup\{j\}$. \\
We know that $[\mathbf{A}_{o,o}]_{i,k}\neq0$ if and only if $i\in\ch(k)$, while $[\Tilde{\mathbf{A}}_{o,o}]_{i,k}$ is always different from 0. Therefore, the condition in \eqref{eq:graph:eq:cond} fails if and only if there is an $i\in\ch(l)\setminus\ch(j)$.
\end{proof}

\begin{proof}[Proof of \cref{thm:graph:id}]
    In \cref{lem:graph:id}, we have shown that the statement is true if $\sigma$ is a transposition. We now assume that $\sigma=\sigma_n\cdots\sigma_1$ where $\sigma_s$ is a transposition for every $s\in\{1, \dots, n\}$, such that if $\sigma_s(j)\neq j$ for some $j$ then also $\sigma(j)\neq j$. This can be done without loss of generality since every permutation can be written this way.
    
    We can see that the condition is sufficient from \cref{lem:graph:id}; if it is satisfied, the support of $\mA^\sigma$ cannot change. In order to prove that the condition is also necessary, we need to verify that if there is a $\sigma_s$ such that $[\mA^{\sigma_s}]_{i,k}\neq0$ then $[\mA^{\sigma}]_{i,k}\neq0$ as well. Again, we will prove the result for a pair of observed variables $i$ and $k$ since the result follows the same way when considering latent variables.

    In particular, we are going to prove the following
    \begin{equation}
        [\mathbf{A}^{\sigma}_{o,o}]_{i,k}=
        [\mathbf{A}_{o,o}]_{i,k}+\sum_s\mathbf{1}_{\ch(l_s)}(i)c^{j_s}_{l_s,i}[I-\mathbf{A}_{o,o}]_{j_s,k}+r_\sigma,
    \end{equation}
    where $\sigma_s$ is the transposition that swaps $j_s$ with $l_s$ and $r_\sigma$ is either 0 or a polynomial of degree at least two in $\R[\GG_\mA]$. We proceed by induction on $n$. If $n$ is equal to $1$, then we apply \cref{lem:adj:mat} with $r_\sigma=0$. 
    
    In order to proceed with the induction step, let us define $\sigma_{[s]}=\sigma_s\cdots\sigma_1$ so that we can write $\sigma$ as $\sigma_{k+1}\cdot\sigma_{[k]}$. From the induction, we know that 
    \begin{equation*}
        [\mathbf{A}^{\sigma_{[n]}}_{o,o}]_{i,k} =[\mathbf{A}_{o,o}]_{i,k}+\sum_{s=1}^n\mathbf{1}_{\ch(l_s)}(i)c^{j_s}_{l_s,i}[I-\mathbf{A}_{o,o}]_{j_s,k}+r_{\sigma_{[n]}},
    \end{equation*}
    and by construction we can write $\mB^{\sigma}=\mB^{\sigma_{[n]}}+C^{(n+1)}$ where 
    \begin{equation}
        C^{(n+1)}_{i,k}=
            \begin{cases}
                \displaystyle\sum_{Q\in\mathcal{P}_{\setminus j_{n+1}}(l_{n+1},i)}a^Q         \quad&\emph{if }\,k = j_{n+1}\,\emph{ and }\, i\in \de^{\GG_{\setminus j}}_o(l_{n+1}),\\
                0          \quad&\emph{otherwise}.\\
            \end{cases}
    \end{equation}
    Following the same steps of the proof of \cref{lem:adj:mat}, with the only difference of using $\mB^\sigma$ in place of $\tmB'$, and  $\mB^{\sigma_{[n]}}$ in place of $\mB'$ we obtain
    \begin{equation}
    \label{eq:A:sigma}
        \begin{aligned}
            [\mA^\sigma]_{i,k}=[\mA^{\sigma_{[n]}}]_{i,k}+\sum_{s}(-1)^{s+j_{n+1}}C^{(n+1)}_{s,j_{n+1}}\det([\mB^{\sigma_{[n]}}_o]_{/(k,s),/(i,j_{n+1})}).
        \end{aligned}
    \end{equation}
    Thus, the only thing that is left to prove is that the last term on the right-hand side of the equation is equal to
    $\mathbf{1}_{\ch(l_{n+1})}(i)c^{j_{n+1}}_{l_{n+1},i}[I-\mathbf{A}_{o,o}]_{j_{n+1},j}+r_{\sigma_{n+1}}
    $. In order to do so note that using the same formula as above we can write $\mB^{\sigma_{[n]}}_o=\mB^{\sigma_{[n-1]}}_o+C^n$, and thus $$\det([\mB^{\sigma_{[n]}}_o]_{/(k,s),/(i,j_{n+1})})=\det([\mB^{\sigma_{[n-1]}}_o]_{/(k,s),/(i,j_{n+1})})+r'_{\sigma_{[n-1]}}.$$    
    Plugging the above equation in \eqref{eq:A:sigma}, concludes the proof, following the same steps as in the proof of \cref{lem:adj:mat}.    
\end{proof}

\begin{proof}[Proof of \cref{thm:tot:eff:known:graph}]
    The only difference with respect to \cref{thm:tot:eff} is the condition $$\ch(l)\setminus\ch(k_1)=\emptyset,\quad\forall k_1\in\pa(k)\cup\{k\},$$ that comes from \cref{thm:graph:id}. Indeed, if this condition is not satisfied, then the permutation that swaps the columns corresponding to $k$ and $l$, cannot result in a model in $\R^{\GG}$.
\end{proof}

\begin{proof}[Proof of \cref{thm:dce:known:graph}]
As in the proof of \cref{thm:tot:eff:known:graph}, the difference from the conditions of \cref{thm:dce:id} and \cref{thm:dce:known:graph} is condition \eqref{eq:dce:cond:1}. This is, again, a direct consequence of \cref{thm:graph:id}.
\end{proof}

\begin{proof}[Proof of \cref{thm:mixing:matrix:known}]
The mixing matrix $\mathbf{B}_o$ is identifiable if and only if all of its parameters are. Hence, the result is a direct consequence of \cref{thm:tot:eff:known:graph}.
\end{proof}

\begin{example}[Examples \ref{ex:iv:unknown:graph} and \ref{ex:iv:known:graph} continued]
\label{app:example:iv}
Here, we report the mixing matrices corresponding to the two models that are compatible with the observed distribution
\begin{equation*}
    \mathbf{B}' = \begin{bmatrix}
        1 & 0 & 0 & 0\\
        b_{TI} & 1 & 1 & 0\\
        b_{TI}b_{YT} & b_{YT} & b_{YT}+b_{LY} & 1
    \end{bmatrix}, \quad
    \Tilde{\mathbf{B}}' = \begin{bmatrix}
        1 & 0 & 0 & 0\\
        b_{TI} & 1 & 1 & 0\\
        b_{TI}b_{YT} & b_{YT}+b_{LY} & b_{YT} & 1
    \end{bmatrix},    
\end{equation*}
together with the corresponding adjacency matrices
\begin{equation*}
    \mathbf{A}_{o, o} = \begin{bmatrix}
        1 & 0 & 0 \\
        b_{TI} & 1 & 0 \\
        0 & b_{YT} & 1
    \end{bmatrix},\quad
    \mathbf{\tilde{A}}_{o, o} = \begin{bmatrix}
        1 & 0 & 0 \\
        b_{TI} & 1 & 0 \\
        -b_{TI}b_{YL} & b_{YT}+b_{YL} & 1
    \end{bmatrix}.
\end{equation*}
    
    \begin{figure}[ht]
        \centering
        \begin{tikzpicture}
        \begin{scope}
        \node (I) at (-2,0) {\large $\GG_{IV}$:};
        \node[draw, circle, inner sep=2pt, minimum size=0.8cm] (I) at (-1,0) {$I$};
        \node[draw, circle, inner sep=2pt, minimum size=0.8cm] (T) at (2,0) {$T$};
        \node[draw, circle, inner sep=2pt, minimum size=0.8cm] (Y) at (5,0) {$Y$};
        \node[draw, circle, fill=gray!40, inner sep=2pt, minimum size=0.8cm] (H) at (3.5, 1) {$L$};
    
        \draw[->] (I) --node[below,blue] {$b_{TI}$} (T);
        \draw[->] (T) --node[below,red] {$b_{YT}$} (Y);
        \draw[->] (H) --node[left, pos = 0.1] {$1\quad$} (T);
        \draw[->] (H) --node[right, pos = 0.1, red] {$\quad b_{YL}$} (Y);
    \end{scope}
        
    \begin{scope}[xshift=8cm]
        \node (I) at (-2,0) {\large $\tilde{\GG}_{IV}$:};
        \node[draw, circle, inner sep=2pt, minimum size=0.8cm] (I) at (-1,0) {$I$};
        \node[draw, circle, inner sep=2pt, minimum size=0.8cm] (T) at (2,0) {$T$};
        \node[draw, circle, inner sep=2pt, minimum size=0.8cm] (Y) at (5,0) {$Y$};
        \node[draw, circle, fill=gray!40, inner sep=2pt, minimum size=0.8cm] (H) at (3.5, 1) {$L$};
    
        \draw[->] (I) --node[below,blue] {$b_{TI}$} (T);
        \draw[->] (I) to[bend right = 40]node[below, red] {$-b_{TI}b_{YL}$} (Y);
        \draw[->] (T) --node[below,red] {$b_{YT}+b_{YL}$} (Y);
        \draw[->] (H) --node[left, pos = 0.1] {$1\quad$} (T);
        \draw[->] (H) --node[right, pos = 0.1, red] {$\quad -b_{YL}$} (Y);
        
    \end{scope}
    \end{tikzpicture}
        \caption{The graphs corresponding to the two models.}
        \label{appendix:fig:iv:graphs}
    \end{figure}
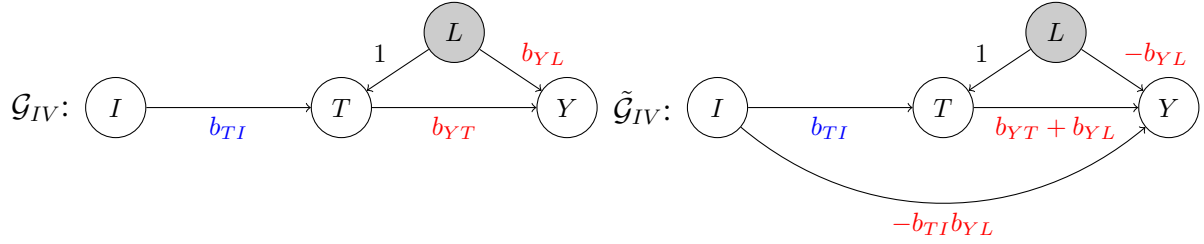

    The two models are depicted in \cref{appendix:fig:iv:graphs}. A parameter is generically identifiable without knowledge of the graph if it is the same for both models. Since there is only one model that is compatible with $\GG_{IV}$, all the parameters are generically identifiable with knowledge of the graph. These results are summarized in \cref{tab:summary:iv:graph}.
    
    \begin{table}[ht]
    \centering
    \caption{Summary of the identifiable parameters in the IV graph.}
    \label{tab:summary:iv:graph}
    \vspace{0.3cm}
    \begin{tabular}{|c|c|c|c|c|}
        \hline
        & \multicolumn{2}{c|}{Known DAG} & \multicolumn{2}{c|}{Unknown DAG} \\
        \cline{2-5}
        & TCE & DCE & TCE & DCE \\
        \hline
        $I \to T$ & \ding{51} & \ding{51} & \ding{51}  &\ding{51}  \\
        \hline
        $I \to Y$ &\ding{51}  &\ding{51}  &\ding{51} & \ding{55}\\
        \hline
        $T \to Y$ &\ding{51}  &\ding{51}  & \ding{55} & \ding{55}\\
        \hline
    \end{tabular}
\end{table}
\end{example}

\subsection{Proofs of \cref{subsec:cert}}
\label{subsec:proof:cert}

\begin{theorem}
\label{thm:cert:dce}
    Algorithms \ref{alg:tce:known}, \ref{alg:tce:unknown}, \ref{alg:dce:known}, \ref{alg:dce:unknown}, \ref{alg:matrix:known}, and \ref{alg:matrix:unknown} are sound and complete for solving the corresponding identifiability queries, as summarized in \cref{tab:summary:id}.
    
    Moreover, the computational complexity the algorithms is $\mathcal{O}(p_l(p_o^2+|E|))=\mathcal{O}(p^3)$.    
\end{theorem}

\begin{table}[ht]
    \centering
    \caption{Summary of the identification algorithms, with the corresponding identifiability queries.}
    \label{tab:summary:id}
    \vspace{0.3cm}
    \begin{tabular}{|c|c|c|c|c|}
        \hline
        & \multicolumn{2}{c|}{Known DAG} & \multicolumn{2}{c|}{Unknown DAG} \\
        \cline{2-5}
        & TCE & DCE & TCE & DCE \\
        \hline
        Given Pair & \cref{alg:tce:known} &\cref{alg:dce:known} &\cref{alg:tce:unknown} &\cref{alg:dce:unknown} \\
        \hline
        Complete Matrix &\cref{alg:matrix:known} &\cref{alg:matrix:known} &\cref{alg:matrix:unknown} & \cref{alg:matrix:unknown}\\
        \hline
    \end{tabular}
\end{table}

\begin{proof}[Proof of \cref{thm:cert} and \cref{thm:cert:dce}]

In order to prove that correctness of \cref{alg:tce:known}, it is enough to show that if $\de_o(l)=\de_o(j)$ then $j < k$ for every $k\in\ch(l)\setminus\{j\}$. Assume by contradiction that there is $k < j$ in $\ch(l)$. Then $k$ cannot be a descendent of $j$ and so $\de_o(l)\neq\de_o(j)$. This implies that the total causal effect is identifiable if the condition at line $7:$ in \cref{alg:tce:known} is not satisfied for any latent variable. The rest of the algorithm is just a translation of the conditions in \cref{thm:tot:eff:known:graph}.

The nodes can be arranged in a topological order in $\mathcal{O}(p+|E|)$ time with one run of depth-first search, see e.g., \citet[\S22.4]{cormen:2009}. Computing the descendants of a node can be done again with depth-first search in the graph in which only the observed nodes are considered since we assume no edges from observed to latent variables. Therefore, it has a cost of $\mathcal{O}(p_o+|E|)$, sorting $\ch(l_k)$ in line $5$ of the algorithm is $\mathcal{O}(p_olog(p_o))$, while computing the parents of a node can be done in $\mathcal{O}(p_o^2)$. Therefore, the internal loop that starts at line $4$ has the complexity of $\mathcal{O}(p_o^2+|E|)$ and this is repeated at most $p_l$ times. Hence, the loop costs $\mathcal{O}(p_l(p_o^2+|E|))$. The final cost is $\mathcal{O}(p_l(p_o^2+|E|)+p+|E|)=\mathcal{O}(p_l(p_o^2+|E|))=\mathcal{O}(p^3)$.

The result for \cref{thm:cert:dce} follows in the same way. Note that a consequence of \cref{lem:isomorphism} is that when one is interested in identifying the whole matrix, it is not necessary to distinguish between direct and total causal effects.
\end{proof}

\begin{algorithm}
    \caption{Total Causal Effect Identification without knowledge of the graph}
    \label{alg:tce:unknown}
    \textbf{INPUT:} $\mathcal{V}=\mathcal{O}\cup\mathcal{L}, \GG, \{\ch(i) \mid i \in \mathcal{V}\}, (j, i)$
    \begin{algorithmic}[1]
    \STATE $\text{ID} \gets \text{TRUE}$ 
    \STATE Sort $\mathcal{V}$ according to an ascending topological order
    \STATE Compute $\de_o(j)$    
    \WHILE{$\text{ID} == \text{TRUE}$ \AND $|\mathcal{L}|>0$}
        \STATE $l\gets \mathcal{L}[1]$\COMMENT{The first element in the list}
        \STATE Sort $\ch(l)$ according to the topological order defined in step 2 
        \IF{$\ch(l)[1] = j = j$}
            \STATE Compute $\de^{\GG_{\setminus j}}_o(l)$
            \IF[{\color{gray}\eqref{eq:tce:unknown:2}}\color{black}]{$i\in\de^{\GG_{\setminus j}}_o(l)$}
                \STATE Compute $\de_o(l)$
                \IF{$\de_o(l) =\de_o(j)$}   
                    \STATE $\text{ID} \gets \text{FALSE}$ \COMMENT{{\color{gray}\eqref{eq:tce:unknown:1}}}                                    
                \ENDIF
            \ENDIF
        \ENDIF
        \STATE $\mathcal{L}\gets \mathcal{L}\setminus\{l\}$
    \ENDWHILE
    \STATE \textbf{RETURN:} ID
    \end{algorithmic}
    \end{algorithm}

\begin{algorithm}
\caption{Direct Causal Effect Identification with knowledge of the graph}
\label{alg:dce:known}
\textbf{INPUT:} $\mathcal{V} = \mathcal{O} \cup \mathcal{L}, \GG, \{\ch(i) \mid i \in \mathcal{V}\}, (j, i)$
\begin{algorithmic}[1]
\STATE $\text{ID} \gets \text{TRUE}$
\STATE Sort $\mathcal{V}$ according to an ascending topological order

\FORALL{$l\in \mathcal{L}$}
    \STATE Sort $\ch(l)$ according to the topological order defined in step 2
    \STATE Compute $\de_o(l)$
    \IF[{\color{gray}\eqref{eq:dce:cond:3}}\color{black}]{$i\in\ch(l)$}
        \FORALL{$k\in\ch(j)\cup\{j\}$}
            \STATE Compute $\de_o(k)$
            \IF[{\color{gray}\eqref{eq:dce:cond:2}}\color{black}]{$\de_o(k)=\de_o(l)$}
                \STATE ID $\gets$ FALSE
                \STATE Compute $\pa(j)$
                \FORALL{$k_1\in\pa(j)\cup\{j\}$}
                    \IF[{\color{gray}\eqref{eq:dce:cond:1}}\color{black}]{$\ch(l)\setminus\ch(k_1)\neq\emptyset$}
                        \STATE ID $\leftarrow$ TRUE
                    \ENDIF
                \ENDFOR
            \ENDIF
        \ENDFOR
    \ENDIF
\ENDFOR
\STATE \textbf{RETURN:} ID
\end{algorithmic}
\end{algorithm}

\begin{algorithm}
\caption{Direct Causal Effect Identification without knowledge of the graph}
\label{alg:dce:unknown}
    \textbf{INPUT:} $\mathcal{V} = \mathcal{O} \cup \mathcal{L}, \{\ch(i) \mid i \in \mathcal{V}\}, (j, i)$
    \begin{algorithmic}[1]
    \STATE $\text{ID} \gets \text{TRUE}$
    \STATE Sort $\mathcal{V}$ according to an ascending topological order    
    \FORALL{$l\in \mathcal{L}$}
        \STATE Sort $\ch(l)$ according to the topological order defined in step 2
        \STATE Compute $\de_o(l)$
        \IF[{\color{gray}\eqref{eq:dce:id:cond:1}}\color{black}]{$i\in\ch(l)$}
            \FORALL{$k\in\ch(j)\cup\{j\}$}
                \STATE Compute $\de_o(k)$
                \IF[{\color{gray}\eqref{eq:dce:id:cond}}\color{black}]{$\de_o(k)=\de_o(l)$}
                    \STATE ID $\gets$ FALSE                
                \ENDIF
            \ENDFOR
        \ENDIF
    \ENDFOR
\STATE \textbf{RETURN:} ID
\end{algorithmic}
\end{algorithm}

\begin{algorithm}
\caption{Matrix Identification with knowledge of the graph}
\label{alg:matrix:known}
\textbf{INPUT:} $\mathcal{V} = \mathcal{O} \cup \mathcal{L}, \GG, \{\ch(i) \mid i \in \mathcal{V}\}$
\begin{algorithmic}[1]
\STATE $\text{ID} \gets \text{TRUE}$
\STATE Sort $\mathcal{V}$ according to an ascending topological order
\FORALL{$l\in \mathcal{L}$}
    \STATE Sort $\ch(l)$ according to the topological order defined in step 2
    \STATE $j \gets \ch(l)[1]$
    \STATE Compute $\de_o(l)$ and $\de_o(j)$
    \IF{$\de_o(j)=\de_o(l)$}                
        \STATE $\text{ID} \gets \text{FALSE}$   
        \FORALL{$k\in\ch(j)$}
            \FORALL{$k_1\pa(k)\cup\{k\}$}
                \IF{$\ch(l)\setminus\ch(k_1)\neq\emptyset$}
                \STATE $\text{ID} \gets \text{TRUE}$
                \ENDIF
            \ENDFOR
        \ENDFOR
    \ENDIF
\ENDFOR
\STATE \textbf{RETURN:} ID
\end{algorithmic}
\end{algorithm}

\begin{algorithm}
\caption{Matrix Identification without knowledge of the graph}
\label{alg:matrix:unknown}
\textbf{INPUT:} $\mathcal{V} = \mathcal{O} \cup \mathcal{L}, \{\ch(i) \mid i \in \mathcal{V}\}, \GG$
\begin{algorithmic}[1]
\STATE $\text{ID} \gets \text{TRUE}$
\STATE Sort $\mathcal{V}$ according to an ascending topological order
\FORALL{$l\in \mathcal{L}$}
    \STATE Sort $\ch(l)$ according to the topological order defined in step 2
    \STATE $j \gets \ch(l)[1]$
    \STATE Compute $\de_o(l)$ and $\de_o(j)$
    \IF{$\de_o(j)=\de_o(l)$}                
        \STATE ID $\gets$ FALSE        
    \ENDIF        
\ENDFOR
\STATE \textbf{RETURN:} ID
\end{algorithmic}
\end{algorithm}
\section{Details on the Experimental Setting and Additional Experiments}
\label{app:exp}
\subsection{Identification}
\subsubsection{Mixing Matrix Identification Experiments}
We used \cref{alg:matrix:known} and \cref{alg:matrix:unknown} on randomly generated graphs. In \cref{fig:id:exp}, we report the percentage of graphs in which all the causal effects are identifiable for graphs of size 10 and 100. The graphs are generated according to an Erd\H{o}s-Rényi model in which we ensure that the graph we sample is canonical. The probability of acceptance of an edge is plotted on the $x$-axis. For each setup, we randomly sample $500$ graphs. 

For the case in which the graph is known, we find the same qualitative behavior observed in \cref{subsec:id:exp}. In contrast, when we do not assume the graph to be known, the probability that all the parameters in the graph are identifiable drops drastically.
We found the same qualitative behavior with larger graphs.

The average running time for different graph sizes is shown in \cref{fig:run:time:matrix}.

\begin{figure}
    \centering
        \includegraphics[scale=0.3]{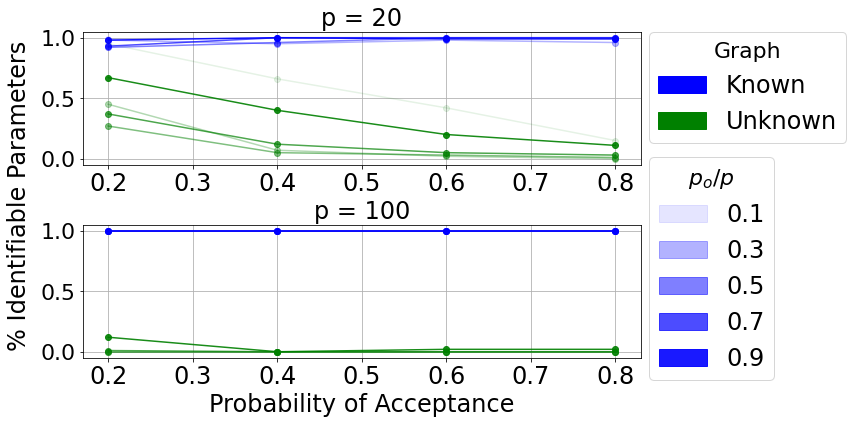}
        \label{fig:perc:id:matrix}
        \caption{On the $x$-axis, the probability of acceptance of an edge.
The $y$-axis shows the percentage of graphs in which all the parameters are identifiable}
\end{figure}
\begin{figure}    
        \centering
        \includegraphics[scale=0.3]{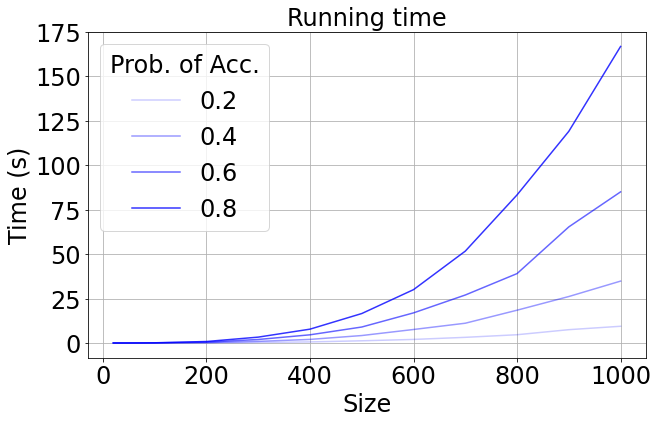}                        
    \caption{On the $x$-axis, the size of the graph. On the $y$-axis, the
average running time in seconds, $p_o/p$ is fixed to 0.5.}
    \label{fig:run:time:matrix}
\end{figure}
\subsection{Estimation}
\label{app:subsec:est}

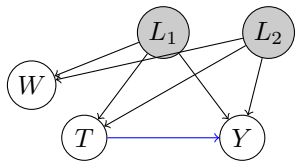
\begin{figure}[ht]
    \centering
    \begin{tikzpicture}[scale=0.7]
        \node[draw, circle, fill=gray!40, inner sep=2pt, minimum size=0.6cm] (z1) at (1,0) {$L_1$};
        \node[draw, circle, fill=gray!40, inner sep=2pt, minimum size=0.6cm] (z2) at (3,0) {$L_2$};
        
        \node[draw, circle, inner sep=2pt, minimum size=0.6cm] (T) at (-0.5,-2) {$T$};
        \node[draw, circle, inner sep=2pt, minimum size=0.6cm] (Y) at (2.5,-2) {$Y$};
        
        \node[draw, circle, inner sep=2pt, minimum size=0.6cm] (W) at (-1.5,-1) {$W$};
        
        \draw[black, ->] (z1) to (T);
        \draw[black, ->] (z1) to (Y);
        \draw[black, ->] (z1) to (W);
        
        \draw[black, ->] (z2) to (T);
        \draw[black, ->] (z2) to (Y);
        \draw[black, ->] (z2) to (W);

        \draw[blue, ->] (T) to (Y);        
    \end{tikzpicture}
    \caption{$\mathcal{G}_4$}
\end{figure}
\subsubsection{Sample size vs relative error}
In this subsection we present additional experiments on synthetic data to illustrate the performance of GRICA algorithm in comparison with the RICA algorithm implemented by \citet{salehkaleybar:2020} and Cross-Moment method by \citet{kivva:2023}.

\paragraph{Experimental setup.}
All the experiments in this subsection are done on the synthetic data generated according to the specific causal structure established for it. To generate synthetic data we specify all exogenous noises to be i.i.d., and select all non-zero entries within the matrix $\mathbf{A}$ through uniform sampling from $[-1, -0.5]\cup [0.5, 1]$. In the following, we display the results for the setups described in \cref{tab:synthetic_data}.

\begin{table}[ht]    
    \centering
    \caption{Summary of the experimental setups.}
    \label{tab:synthetic_data}
    \vspace{0.3cm}
    \begin{tabular}{|c|c|c|c|c|c|}
        \hline
        \textbf{Figure} & \textbf{Causal Graph} & \multicolumn{3}{c|}{\textbf{Distribution}} & \textbf{Parameter of Interest} \\
        \cline{3-5}
        & & \textbf{Family} & $\mu$ & $scale$ &\\
        \hline
        \ref{fig:laplace:sample size vs err:graph4} & $\GG_4$ & Laplace & 0 & 1 & $T\to Y$ \\
        \hline
        \ref{fig:laplace:sample size vs err:iv} & IV, \cref{fig:IV} & Laplace & 0 & 1 & $T_2\to Y$ \\
        \hline
        \ref{fig:exponential:sample size vs err:graph1} & $\GG_1$ & Exponential & - & 1 & $T\to Y$ \\
        \hline
        \ref{fig:exponential:sample size vs err:graph2} & $\GG_2$ & Exponential & - & 1 & $T\to Y$ \\
        \hline        
        \ref{fig:exponential:sample size vs err:graph4} & $\GG_4$ & Exponential & - & 1 & $T\to Y$ \\    
        \hline
        \ref{fig:exponential:sample size vs err:iv} & IV, \cref{fig:IV}& Exponential & - & 1 & $T_2\to Y$ \\
        \hline
    \end{tabular}
    
\end{table}

On the figures, we plot the average relative error over 10 independent experiments, with its standard deviation visualized with a transparent area filled with the respective color.
From the experiments, we see that GRICA outperforms other methods in all setups, except the one for the graph $\GG_1$. Note that for this specific graph, the Cross-Moment method gives us an exact statistical solution that requires computing high-order moments. This results in a less statistically robust estimation when the sample size is small, but it is more accurate if the sample size is large enough.

\begin{figure}
\centering
\begin{subfigure}[t]{0.45\textwidth}
    \centering
    \includegraphics[width=0.9\textwidth]{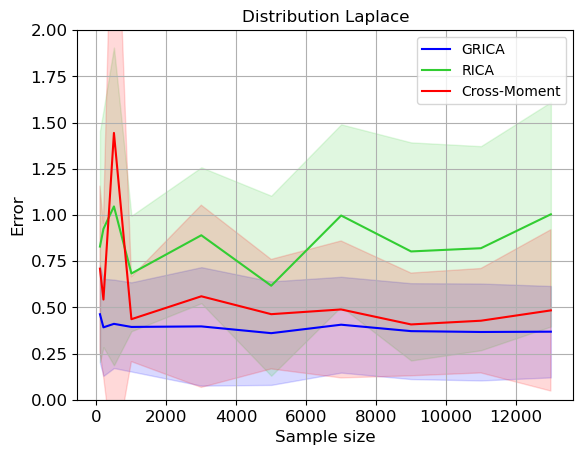}
    \caption{$\GG_4$}
    \label{fig:laplace:sample size vs err:graph4}
\end{subfigure}
\hfill
\begin{subfigure}[t]{0.45\textwidth}
    \centering
    \includegraphics[width=0.9\textwidth]{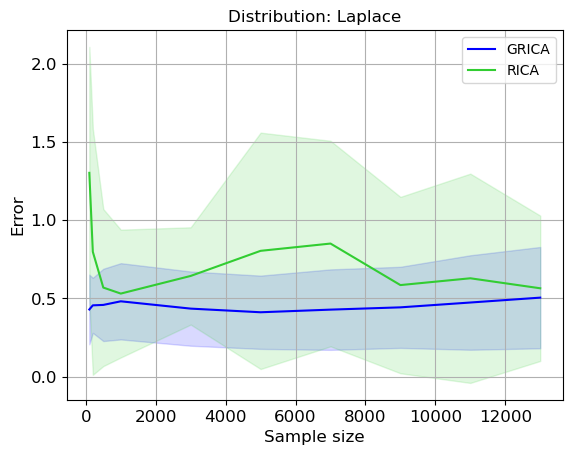}
    \caption{Underspecified instrument}
    \label{fig:laplace:sample size vs err:iv}
\end{subfigure}

\caption{Laplace distribution}
\label{fig:laplace:sample size vs err}
\end{figure}

\begin{figure}
\centering
\begin{subfigure}[t]{0.45\textwidth}
    \centering
    \includegraphics[width=0.9\textwidth]{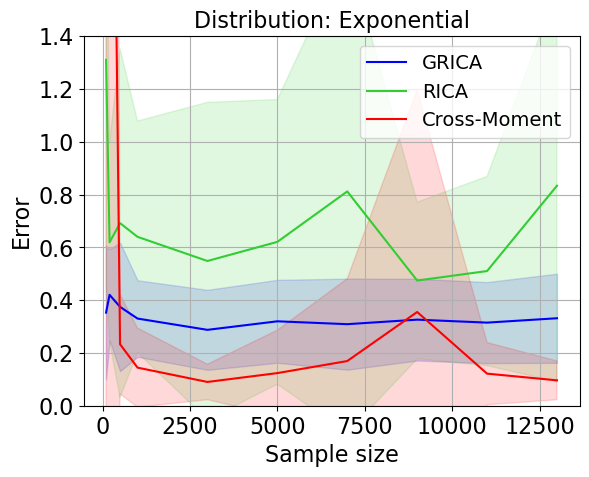}
    \caption{$\GG_1$}
    \label{fig:exponential:sample size vs err:graph1}
\end{subfigure}
\hfill
\begin{subfigure}[t]{0.45\textwidth}
    \centering
    \includegraphics[width=0.9\textwidth]{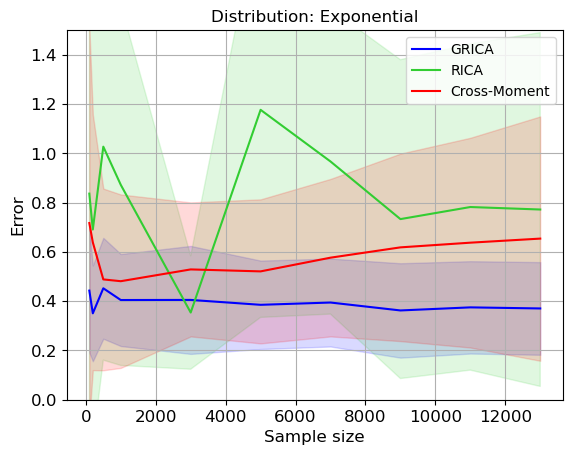}
    \caption{$\GG_2$}
    \label{fig:exponential:sample size vs err:graph2}
\end{subfigure}
\hfill
\begin{subfigure}[t]{0.45\textwidth}
    \centering
    \includegraphics[width=0.9\textwidth]{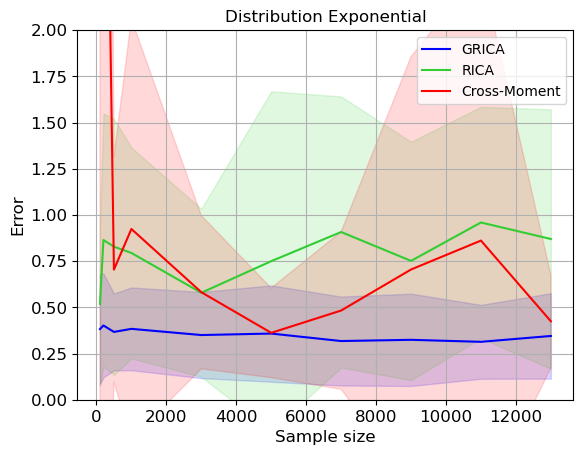}
    \caption{$\GG_4$}
    \label{fig:exponential:sample size vs err:graph4}
\end{subfigure}
\hfill
\begin{subfigure}[t]{0.45\textwidth}
    \centering
    \includegraphics[width=0.88\textwidth]{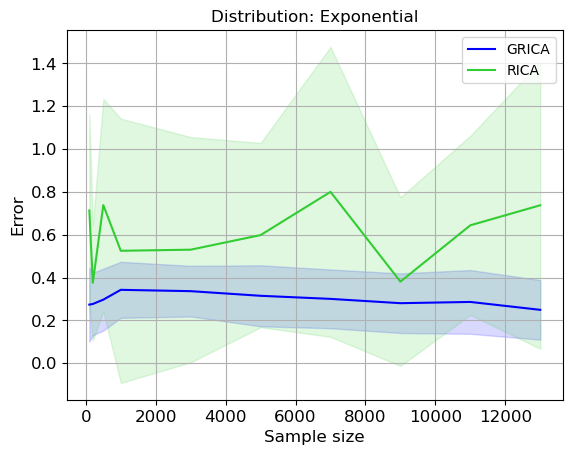}
    \caption{Underspecified instrument}
    \label{fig:exponential:sample size vs err:iv}
\end{subfigure}

\caption{Exponential distribution}
\label{fig:exponential:sample size vs err}
\end{figure}

\subsubsection{Model misspecification}
Here we consider the robustness of GRICA algorithm when data generation process is not linear SEM, but close to it. For these experiments the data is generated with the following machanism:
\[
\begin{cases}
    & V_j = \tanh{\left(\sum_{i=1}^{p}\mathbf{A}_{j, i}V_i\right)} + N_j, \text{ if } V_j\neq Y,\\
    & V_j = \tanh{\left(\sum_{i=1, i\neq k}^{p}\mathbf{A}_{j, i}V_i\right)} + A_{j, k}V_k +  N_j, \text{ if } V_j=Y \text{ and } V_k = T,
\end{cases}
\]
where $T \rightarrow Y$ is a direct causal effect that we want to estimate.

\paragraph{Experimental setup.}
For each experiment we generate the synthetic data accordingly to the causal structure chosen for it. All exogenous noises are i.i.d., and all non-zero non-diagonal entries within the matrix $\mathbf{A}$ sampled uniformly from $[-1, -0.5]\cup [0.5, 1]$. In the following, we display the results for the setups described in \cref{tab:synthetic_data_misspecification}.

\begin{table}[ht]
    \centering
    \caption{Summary of the experimental setups.}
    \label{tab:synthetic_data_misspecification}
    \vspace{0.3cm}
    \begin{tabular}{|c|c|c|c|c|c|}
        \hline
        \textbf{Figure} & \textbf{Causal Graph} & \multicolumn{3}{c|}{\textbf{Distribution}} & \textbf{Parameter of Interest} \\
        \cline{3-5}
        & & \textbf{Familiy} & $\mu$ & $scale$ &\\
        \hline
        \ref{fig:misspecification:laplace:sample size vs err:graph1} & $\GG_1$ & Laplace & 0 & 1 & $T\to Y$ \\
        \hline
        \ref{fig:misspecification:laplace:sample size vs err:graph2} & $\GG_2$ & Laplace & 0 & 1 & $T\to Y$ \\
        \hline
        \ref{fig:misspecification:exponential:sample size vs err:graph1} & $\GG_1$ & Exponential & - & 1 & $T\to Y$ \\
        \hline
        \ref{fig:misspecification:exponential:sample size vs err:graph1} & $\GG_2$ & Exponential & - & 1 & $T\to Y$ \\
        \hline
    \end{tabular}

\end{table}

\begin{figure}
\centering
\begin{subfigure}[t]{0.45\textwidth}
    \centering
    \includegraphics[width=0.85\textwidth]{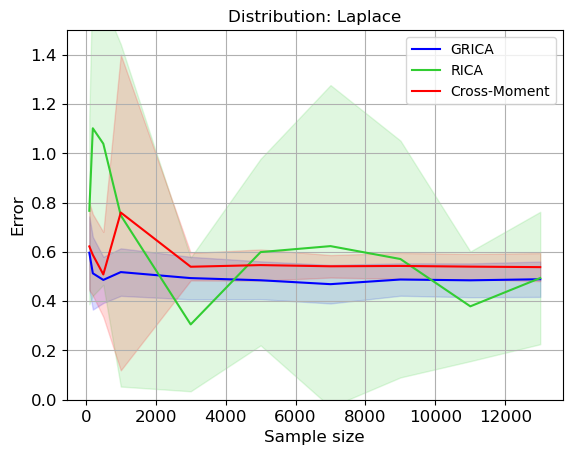}
    \caption{$\GG_1$}
    \label{fig:misspecification:laplace:sample size vs err:graph1}
\end{subfigure}
\hfill
\begin{subfigure}[t]{0.45\textwidth}
    \centering
    \includegraphics[width=0.85\textwidth]{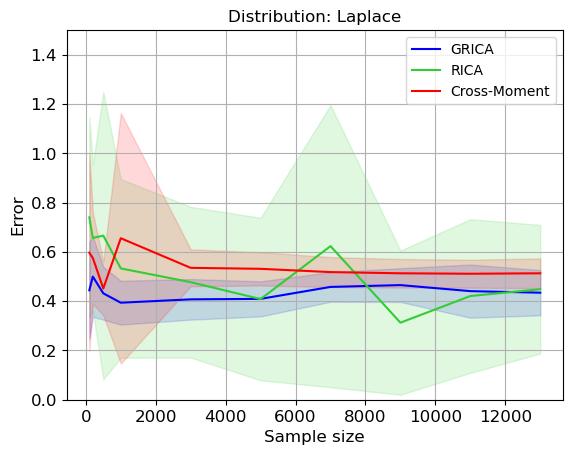}
    \caption{$\GG_2$}
    \label{fig:misspecification:laplace:sample size vs err:graph2}
\end{subfigure}

\caption{Laplace distribution}
\label{fig:misspecification:laplace:sample size vs err}
\end{figure}

\begin{figure}
\centering
\begin{subfigure}[t]{0.45\textwidth}
    \centering
    \includegraphics[width=0.85\textwidth]{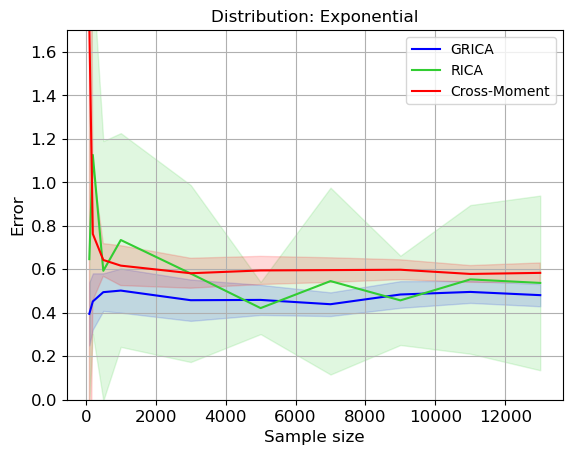}
    \caption{$\GG_1$}
    \label{fig:misspecification:exponential:sample size vs err:graph1}
\end{subfigure}
\hfill
\begin{subfigure}[t]{0.45\textwidth}
    \centering
    \includegraphics[width=0.85\textwidth]{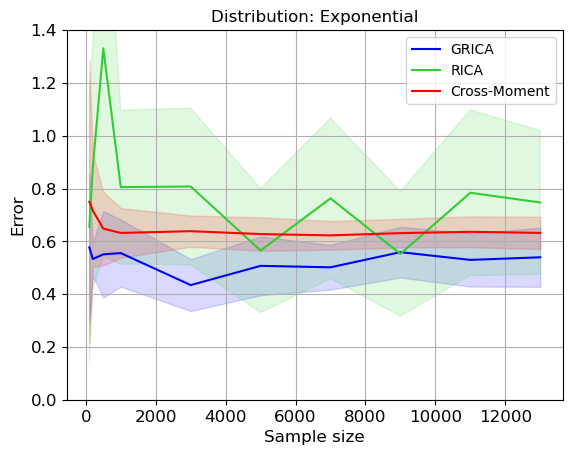}
    \caption{$\GG_2$}
    \label{fig:misspecification:exponential:sample size vs err:graph2}
\end{subfigure}

\caption{Exponential distribution}
\label{fig:misspecification:exponential:sample size vs err}
\end{figure}
On the figures, we plot the average relative error over 10 independent experiments, with its standard deviation visualized with transparent area filled with respective color. We observe that GRICA outperforms all other methods for these experiments.

\subsubsection{Relative error vs observations noise}
These experiments illustrate how the variances of certain exogenous noise impact the performance of GRICA algorithm in comparison to the state-of-the-art algorithms for the considered settings. Herein we considered three main setups. In each of the setup we initialize all the non-zero entries of matrix $\mathbf{A}$ to be equal to 1 and all exogenous noises are modeled 
 as i.i.d. distributions.
\begin{enumerate}
    \item Causal structure: $\GG_1$. The experiment for this causal structure is illustrated in Figure \ref{fig:exponential:snr vs err:graph1}. In this experiment, we initialize all exogenous noises as Exponential distributions. Then we scaled the standard deviation of exogenous noise $N_W$ corresponding to the variable $W$ by a factor displayed on x-axis of the Figure \ref{fig:exponential:snr vs err:graph1} and plotted the performance of GRICA algorithm against the performance of RICA algorithm.
    \item Causal structure: $\GG_3$. The experiment for this causal structure is illustrated in Figure \ref{fig:exponential:snr vs err:graph3}. In this experiment, we initialize all exogenous noises as Exponential distributions. Then we scaled $N_W$ and $N_Z$ by $Ratio$ and $1/Ratio$, respectively, where $Ratio$ is plotted on the x-axis of Figure \ref{fig:exponential:snr vs err:graph1} and plotted the performance of GRICA algorithm against the performance of RICA, Cross-Moment method and method proposed by \citet{tchetgen2020introduction}.
 \item Causal structure: underspecified instrument. The experiments for this causal structure are illustrated in  Figures \ref{fig:laplace:snr vs err:iv} and \ref{fig:exponential:snr vs err:iv}. In Figure \ref{fig:laplace:snr vs err:iv} all exogenous noises are modeled as a Laplace distribution and in Figure \ref{fig:exponential:snr vs err:iv} all exogenous noises are modeled as an Exponential distribution. Then we scaled $N_{L_1}$ and $N_{L_2}$ by a factor displayed on the $x-axis$ of each Figure and estimated the direct causal effect from $T_2$ to $Y$ with GRICA and RICA algorithms. The relative error of the estimations is plotted in the Figures.
\end{enumerate} 

\begin{figure}
\centering
\begin{subfigure}[t]{0.45\textwidth}
    \centering
    \includegraphics[width=0.85\textwidth]{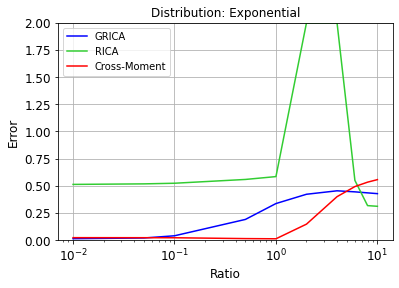}
    \caption{$\GG_1$}
    \label{fig:exponential:snr vs err:graph1}
\end{subfigure}
\hfill
\begin{subfigure}[t]{0.45\textwidth}
    \centering
    \includegraphics[width=0.85\textwidth]{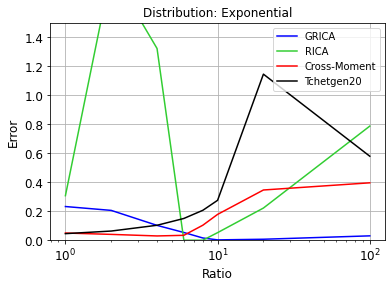}
    \caption{$\GG_3$}
    \label{fig:exponential:snr vs err:graph3}
\end{subfigure}
\caption{Exponential distribution}
\end{figure}

\begin{figure}
\centering
\begin{subfigure}[t]{0.45\textwidth}
    \centering
    \includegraphics[width=0.85\textwidth]{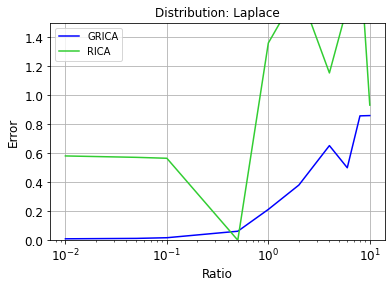}
    \caption{Laplace distribution}
    \label{fig:laplace:snr vs err:iv}
\end{subfigure}
\hfill
\begin{subfigure}[t]{0.45\textwidth}
    \centering
    \includegraphics[width=0.85\textwidth]{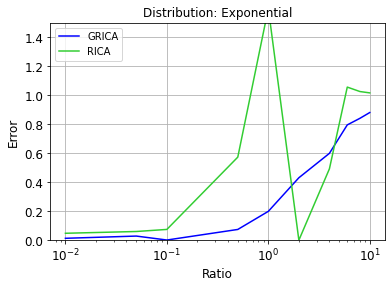}
    \caption{Exponential distribution}
    \label{fig:exponential:snr vs err:iv}
\end{subfigure}
\caption{Underspecified insturment}
\end{figure}

\end{document}